\documentclass[letterpaper]{article} 

\pdfoutput=1

\usepackage{aaai25} 
\usepackage{times}  
\usepackage{helvet}  
\usepackage{courier}  
\usepackage[hyphens]{url}  
\usepackage{graphicx} 
\usepackage{natbib}  
\usepackage{caption} 
\usepackage{newfloat}
\usepackage{listings}
\DeclareCaptionStyle{ruled}{labelfont=normalfont,labelsep=colon,strut=off} 
\newcommand{\oururl}{\url{https://anonymous.4open.science/r/federatedcotraining-B03C/}}
\input{preamble}

\SetCommentSty{textnormal}
\SetKwInOut{Output}{output}
\SetKwInOut{Input}{input}
\DontPrintSemicolon
\SetNoFillComment
\SetAlgoNoLine

\newcommand{\defemph}[1]{\emph{#1}}

\newcommand{\fedct}{\textsc{FedCT}\xspace}
\newcommand{\aimhi}{\fedct}
\newcommand{\cent}{\textsc{Centralized}\xspace}
\newcommand{\dpfedct}{\textsc{DP-FedCT}\xspace}
\newcommand{\dpaimhi}{\dpfedct}
\newcommand{\fedavg}{\textsc{FedAvg}\xspace}
\newcommand{\dd}{\textsc{DD}\xspace}
\newcommand{\dpfl}{\textsc{DP-FedAvg}\xspace}
\newcommand{\pate}{\textsc{PATE}\xspace}
\newcommand{\dppate}{\textsc{DP-PATE}\xspace}
\newcommand{\feddistill}{\textsc{FedDistill}\xspace}
\newcommand{\fedmd}{\textsc{FedMD}\xspace}

\theoremstyle{plain}
\newtheorem{theorem}{Theorem}
\newtheorem{lemma}[theorem]{Lemma}
\newtheorem{remark}[theorem]{Remark}
\newtheorem{proposition}{Proposition}
\newtheorem{definition}[theorem]{Definition}
\newtheorem{corollary}{Corollary}

\newtheorem{notation}[theorem]{Notation}
\newtheorem{conjecture}[theorem]{Conjecture}
\newtheorem{assumption}[theorem]{Assumption}
\newtheorem{observation}[theorem]{Observation}
\newtheorem{fact}[theorem]{Fact}
\newtheorem{claim}[theorem]{Claim}
\newtheorem{problem}[theorem]{Problem}
\newtheorem{open}[theorem]{Open Problem}
\newtheorem{hypothesis}[theorem]{Hypothesis}
\newtheorem{question}[theorem]{Question}
\newtheorem{case}{Case}
\newtheorem*{proposition*}{Proposition}
\newtheorem*{lemma*}{Lemma}
\newtheorem*{corollary*}{Corollary}

\newcommand{\algo}{\mathcal{A}}
\newcommand{\risk}{\varepsilon}
\newcommand{\model}{h}
\newcommand{\modelspace}{\mathcal{H}}
\newcommand{\aggmodel}{\overline{\model}}
\newcommand{\loss}{\ell}

\DeclareMathOperator{\agg}{agg}

\newcommand{\Prob}[2]{\mathop{{}\mathbb{P}_{#1}} \left ( #2 \right ) }

\newcommand{\Dcal}{\mathcal{D}}
\newcommand{\RR}{\mathbb{R}}
\newcommand{\NN}{\mathbb{N}}

\newcommand{\Xcal}{\mathcal{X}}
\newcommand{\Ycal}{\mathcal{Y}}
\newcommand{\bigo}{\mathcal{O}}

\title{Little is Enough: Boosting Privacy by Sharing Only Hard Labels in Federated Semi-Supervised Learning}

\usepackage[capitalize,noabbrev]{cleveref}
\usepackage[textsize=tiny]{todonotes}

\author{%
  Amr Abourayya \textsuperscript{\rm1,\rm2},
  Jens Kleesiek \textsuperscript{\rm1},
  Kanishka Rao \textsuperscript{\rm5},
  Erman Ayday \textsuperscript{\rm4},
  Bharat Rao \textsuperscript{\rm5},\\
  Geoffrey I. Webb \textsuperscript{\rm3},
  Michael Kamp \textsuperscript{\rm1,\rm2,\rm3,\rm5}
}

\affiliations{
    \textsuperscript{\rm 1} Institute for AI in medicine (IKIM),University Hospital Essen, Germany\\
    \textsuperscript{\rm 2} Institute for Neuroinformatics, Ruhr University Bochum, Germany \\
    \textsuperscript{\rm 3} Department of Data Science \& AI, Monash University, Australia \\
    \textsuperscript{\rm 4}  Department of Computer and Data Sciences, Case Western Reserve University, USA \\
    \textsuperscript{\rm 5} Carenostics, USA \\ 
    Amr.Abourayya@uk-essen.de , Michael.Kamp@uk-essen.de    
}

\begin{document}

\maketitle

\begin{abstract}
    In many critical applications, sensitive data is inherently distributed and cannot be centralized due to privacy concerns. A wide range of federated learning approaches have been proposed to train models locally at each client without sharing their sensitive data, typically by exchanging model parameters, or probabilistic predictions (soft labels) on a public dataset or a combination of both. However, these methods still disclose private information and restrict local models to those that can be trained using gradient-based methods. We propose a federated co-training (\fedct) approach that improves privacy by sharing only definitive (hard) labels on a public unlabeled dataset. Clients use a consensus of these shared labels as pseudo-labels for local training. This federated co-training approach empirically enhances privacy without compromising model quality. In addition, it allows the use of local models that are not suitable for parameter aggregation in traditional federated learning, such as gradient-boosted decision trees, rule ensembles, and random forests. 
    Furthermore, we observe that \fedct performs effectively in federated fine-tuning of large language models, where its pseudo-labeling mechanism is particularly beneficial.  Empirical evaluations and theoretical analyses suggest its applicability across a range of federated learning scenarios.
\end{abstract}

\setcounter{secnumdepth}{2} 

\section{Introduction}
\label{sec:introduction}

\textit{Can we train models using distributed sensitive datasets while maintaining data privacy?} Federated learning (\fedavg) ~\citep{mcmahan2017communication} addresses this challenge by collaboratively training a joint model without directly disclosing distributed sensitive data, but instead sharing information from locally trained models. Most approaches share model parameters that are aggregated at a server~\citep{kamp2019black, mcdonald2009efficient, kairouz2021advances}, most prominently federated averaging (\fedavg)~\citep{mcmahan2017communication}. This results in high model performance, but allows an attacker or curious observer to make non-trivial inferences about local data from those model parameters~\citep{ma2020safeguarding} or model updates~\citep{zhu2020deep}. Adding tailored noise to the parameters before sharing improves privacy and provides differential privacy guarantees, as in differentially private \fedavg~\cite{wei2020federated} (\dpfl), but reduces model performance~\citep{xiao2022differentially}. An alternative is to use a form of federated semi-supervised learning that shares information via a public unlabeled dataset, which can improve privacy---most commonly soft labels are shared~\citep{gong2022preserving,lin2020ensemble,jiang2020federated,li2019fedmd}, e.g., as in distributed distillation~\citep{bistritz2020distributed} (\dd). \citep{struppek2023careful} showed, however, that data can be reconstructed also from soft labels. 
In Fig.~\ref{fig:privacy_summary}, we compare the empirical privacy vulnerability of these approaches: Although sharing soft labels (\dd) offers a better privacy-utility trade-off than \dpfl and \fedavg, the gap to ideal privacy (VUL$=0.5$, where membership inference attacks are akin to random guessing) remains significant. 

\begin{figure}
    \centering
    \includegraphics[width=0.95\linewidth]{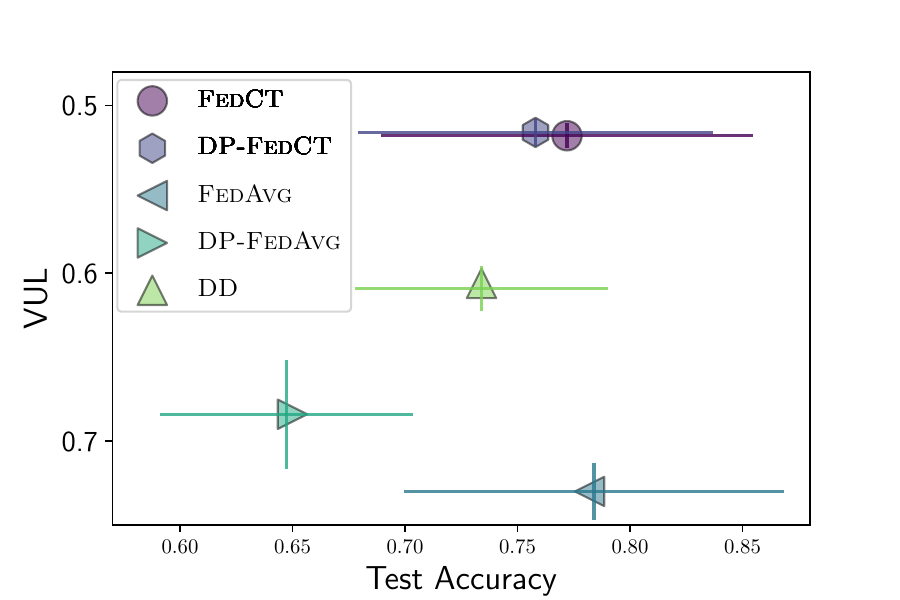}
    \caption{Vulnerability (VUL) to membership inference attacks on the communication of $5$ clients and their test accuracy (avg and std over $5$ datasets). VUL is measured empirically as the success probability of infering membership correctly, $\text{VUL}=0.5$ implies optimal privacy.}
    \label{fig:privacy_summary}
\end{figure}
We propose instead to share definitive class labels (hard labels) to improve privacy: We develop a novel federated co-training approach (\fedct) where clients iteratively share hard labels on an unlabeled public dataset. A server forms a consensus of these predictions, e.g., via 
majority vote~\citep{blum1998combining}, and clients use this consensus as pseudo-labels for the unlabeled dataset in their local training. 

We prove convergence of this approach and empirically show that it provides a favorable privacy-utility trade-off (see Fig.~\ref{fig:privacy_summary}), achieving near-optimal privacy levels. Using a majority vote as consensus mechanism results in a model performance that is at least en par with model parameter and soft-label sharing. In tasks where the initial models already provide good predictions, such as fine-tuning of LLMs, \fedct can even outperform \fedavg. We show how differential privacy guarantees~\citep{ziller2021medical, chaudhuri2019capacity} can be provided for \fedct via a suitable noise mechanism for binary outputs, which not only protect privacy but can improve robustness against poisoning and backdoor attacks~\citep{bagdasaryan2020backdoor, sun2019can}. For that, we provide a novel bound on the sensitivity of hard label sharing for local learning algorithms that are on-average-leave-one-out-stable. Due to the inherent robustness of majority voting~\citep{papernot2016semi}, the added noise in differentially private federated co-training (\dpfedct) does not reduce model performance as much as in \fedavg. As~\citet{bistritz2020distributed} noted, sharing labels can also reduce communication substantially, improving scalability. We confirm this for \fedct, reducing communication over \fedavg by up to two orders of magnitude.

While federated semi-supervised learning performs well on heterogeneous feature distributions, both soft and hard label often underperform on data with heterogeneous label distributions, since local models are unable to provide meaningful predictions for labels they have not observed during training. With mild-to-medium data heterogeneity, label sharing performs well, but its performance drops for pathological distributions. We propose using a qualified majority as a consensus for pathological distributions which improves the performance substantially, making \fedct competitive with \fedavg also for pathological non-iid data.

A positive side-effect of sharing hard labels is that one is no longer limited to models that lend themselves to parameter aggregation as in \fedavg, or gradient-based training as used in soft-label sharing: With \fedct we can train, e.g., interpretable models, such as decision trees~\citep{quinlan1986induction}, XGBoost~\citep{friedman2001greedy, chen2016xgboost}, rule ensembles~\citep{friedman2008predictive}, and Random Forests~\citep{breiman2001random}, as well as mixtures of model types. 

In summary, our contributions are: 
\begin{itemize}
\item  [(i)] An investigation of the privacy-utility trade-off between sharing model parameters, soft labels, and hard labels via a novel federated co-training approach (\fedct);
\item[(ii)] A novel sensitivity bound on sharing hard labels for local learning algorithms that are on-average-replae-one stable;
\item[(iii)] A theoretical analysis of the convergence and sensitivity, and an extensive empirical evaluation of \fedct.
\end{itemize}
\section{Related Work}
\label{sec:related_work}
We briefly discuss closely related work here, focussing on privacy in federated learning and distributed semi-supervised learning and provide a more comprehensive discussion of related work in App.~\ref{app:smifed:discussion}.

\textbf{Privacy in Federated Learning: }
Collaboratively training a model without sharing sensitive data is a key advantage of (horizontal) federated learning~\citep{mcmahan2017communication} which trains local models and aggregates their parameters periodically. Communicating only model parameters, however, does not entirely protect local data: An attacker can make inferences about local data from model parameters~\citep{shokri2017membership, ma2020safeguarding} and model updates~\citep{zhu2020deep}. 
A common defense is perturbing shared information, e.g., applying appropriate clipping and noise before sending model parameters. This can guarantee $\epsilon,\delta$-differential privacy for local data~\citep{wei2020federated} at the cost of model quality. This technique also defends against backdoor and poisoning attacks~\citep{sun2019can}.~\citet{truex2019hybrid} proposes enhancing the privacy of data exchange in traditional distributed algorithms through the use of secure multi-party communication (SMPC) and differential privacy (DP) in applications where the scalability and efficiency limitations of SMPC are irrelevant. 
In these cases, SMPC (as well as homomorphic encryption~\citep{roth2019honeycrisp} and trusted execution environments~\citep{subramanyan2017formal}), offers an additional level of privacy that can be combined with other approaches~\citep[cf. Sec. 4 in ][]{kairouz2021advances}.

\textbf{Distributed semi-supervised learning: }
Semi-supervised learning utilizes both labeled and unlabeled data~\citep{zhou2005tri,rasmus2015semi} for training. In centralized co-training, classifiers are independently trained on distinct feature sets, or views, of labeled data and their consensus on unlabeled data is used as pseudo-labels~\citep{blum1998combining, ullrich2017co}. \citet{papernot2016semi} propose a distributed---but not collaborative---knowledge distillation approach called \pate where teachers are trained distributedly and a consensus of their predictions on the unlabeled data is used to train a student model. We show empirically that \pate and its differentially private 
variant \dppate is outperformed by collaborative approaches.~\citet{bistritz2020distributed} propose to share soft predictions on unlabeled data to reduce communication in federated (and decentralized) deep learning and term their approach distributed distillation (\dd). We compare to \dd, showing that we achieve similar model quality and communication with improved privacy.~\citet{Chen2020FedBEMB} employ knowledge distillation to train a student model based on predictions from a Bayesian model ensemble (FedBE). Similarly, ~\citep{lin2020ensemble}'s FedDF also uses knowledge distillation in a federated context to create a global model by fusing client models. Both FedBE and FedDF require sharing local model parameters and thus have the same privacy issues that \fedavg has.

\section{Federated Semi-Supervised Learning}

\paragraph{Preliminaries: }
We assume learning algorithms $\algo:\Xcal\times\Ycal\rightarrow\modelspace$ that produce models $\model\in\modelspace$ using a dataset $D\subset\Xcal\times\Ycal$ from an input space $\Xcal$ and output space $\Ycal$, i.e., $\model_{t+1}=\algo(D)$, or iterative learning algorithms \citep[cf. Chp. 2.1.4][]{kamp2019black} $\algo:\Xcal\times\Ycal\times\modelspace\rightarrow\modelspace$ that update a model $\model_{t+1}=\algo(D,\model_t)$.
Given a set of $m\in\NN$ clients with local datasets $D^1,\dots,D^m\subset\Xcal\times\Ycal$ drawn iid from a data distribution $\Dcal$ 
and a loss function $\loss:\Ycal\times\Ycal\rightarrow\RR$, the goal is to find a set of local models $\model^{1*},\dots,\model^{m*}\in\modelspace$ that each minimize the risk $\risk(\model)=\mathbb{E}_{(x,y)\sim\Dcal}[\loss(\model(x),y)]$.

In \defemph{centralized learning}, datasets are pooled and $\algo$ is applied to $D=\bigcup_{i\in [m]}D^i$ until convergence, e.g., as full or mini-batch training. Convergence is measured in terms of low training loss, small gradient, or small deviation from previous iterations. In standard \defemph{federated learning}~\citep{mcmahan2017communication}, $\algo$ is applied in parallel for $b\in\NN$ rounds on each client to produce local models $\model^1,\dots,\model^m$ which are centralized and aggregated using an aggregation operator $\agg:\modelspace^m\rightarrow\modelspace$, i.e., $\aggmodel = \agg(\model^1,\dots,\model^m)$. The aggregated model $\aggmodel$ is redistributed to local clients which perform another $b$ rounds of training using $\aggmodel$ as a starting point. This is iterated until convergence of $\aggmodel$. When aggregating by averaging, this method is known as federated averaging (\fedavg). 

In federated semi-supervised learning, a public unlabeled dataset $U$ is available to all clients. Clients can share predictions on $U$ (both hard and soft labels), as well as model parameters. Since sharing model parameters threatens privacy, we consider semi-supervised approaches that share predictions (other approaches are discussed in App.~\ref{app:smifed:discussion}).

\paragraph{A Federated Co-Training Approach: }
\SetKwFor{local}{Locally}{do}{}
\SetKwFor{coord}{At server}{do}{}
\begin{algorithm}[ht]
    \caption{Federated Co-Training (\fedct)}
    \label{alg:aimhi}
    \KwIn{communication period $b$, $m$ clients with local datasets $D^1,\dots,D^m$ and local learning algorithms algorithms $\mathcal{A}^1,\dots\mathcal{A}^m$,  unlabeled shared dataset $U$, total number of rounds $T$}
    \KwOut{final models $h^1_T,\dots,h^m_T$}
    \vspace{0.1cm}
    initialize local models $h_0^1,\dots,h_0^m\enspace$,$\enspace P\leftarrow\emptyset$\\
    \local{at client $i$ at time $t$}{
        $h_{t}^i\leftarrow\mathcal{A}^i(D^i\cup P,h^i_{t-1})$\\
        \If{$t\ \%\ b = b-1$}{
             $L_t^i\leftarrow h_t^i(U)$\\
             send $L_t^i$ to server and receive  $L_t$\\
             $P\leftarrow (U,L_t)$
        }
    }
    \coord{at time $t$}{
        receive local pseudo-labels $L_t^1,\dots,L_t^m$\\
        $L_t\leftarrow \text{consensus}(L_t^1,\dots,L_t^m)$\\
        send $L_t$ to all clients
    }
\end{algorithm}
We propose a federated variant of co-training---originally developed for multi-view semi-supervised learning~\citep{blum1998combining}---that iteratively updates pseudo-labels of $U$ via the consensus of shared predictions. That is, in a communication round $t\in\NN$ each client $i\in [m]$ shares local labels $L^i_t=h^i_t(U)$ (not soft predictions) on $U$ with the server, which produces a consensus labeling $L_t\subset\Ycal$ via an appropriate consensus mechanism. The resulting pseudo-labeled dataset $P=(U,L_t)$ augments local training sets. We call this approach federated co-training (\fedct). Sharing hard labels not only improves privacy, but also allows us to use any supervised learning method for local training.

We describe federated co-training in Algorithm~\ref{alg:aimhi}: at each client $i$, the local model is updated using local dataset $D^i$ combined with the current pseudo-labeled public dataset $P$ (line 4). In a communication round (line 5), the updated model is used to produce improved pseudo-labels $L^i$ for the unlabeled data $U$ (line 6), which are sent to a server (line 7). At the server, when all local prediction $L^1,\dots,L^m$ are received (line 12), a consensus $L$ is formed (line 13) and broadcasted to the clients (14). At the client, upon receiving the consensus labels (line 8), the pseudo-labeled dataset is updated (line 9), and another iteration of local training is performed. 
For classification problems the majority vote is a reliable consensus mechanism~\citep{papernot2016semi}.
\label{sec:convergence:analysis}
%
\paragraph{Convergence Analysis: } The convergence of federated co-training depends on the convergence of the local learning algorithms $\left(\algo^i\right)_{i\in [m]}$. Under the assumption that these algorithms converge on a fixed training set, it remains to show that the training set eventually stabilizes. That is, there exists a round $t_0\in\NN$ such that for all $t> t_0$ it holds that $L_t=L_{t-1}$. For classification problems, this depends on the local training accuracy. If local training accuracy reaches $a_t=1.0$, then the approach trivially converges, since local models will reproduce $L_{t}$ in every subsequent round. This assumption can usually be fulfilled for over-parameterized models. In the following, we show that the approach also converges with high probability, if the training accuracy is $\leq 1$, but increasing with $t$. This is typically the case in federated learning~\citep{zhao2018federated} and federated semi-supervised learning~\citep{papernot2016semi}. To simplify the analysis, we approximate this via a linearly increasing training accuracy.

\begin{restatable}{proposition}{convergence}
For $m\geq 3$ clients with local datasets $D^1,\dots,D^m$ and unlabeled dataset $U$, let $\algo^i$ for $i\in [m]$ be a set of learning algorithms that all achieve a linearly increasing training accuracy $a_t$ for all labelings of $U$, i.e., there exists $c\in\RR_+$ such that $a_t\geq1-c/t$, then there exists $t_0\in\NN$ such that $a_t\geq 1/2$ and \aimhi with majority vote converges with probability $1-\delta$, where
$\delta\leq|U|(4c)^{\frac{m}2}\zeta\left(\frac{m}{2},t_0+1\right)$
and $\zeta(x,q)$ is the Hurwitz zeta function.
\label{prop:convergence}
\end{restatable}
\begin{proof}
(sketch:) If local models are of sufficient quality, then in round $t\geq t_0$, the probability that the consensus labels change, $\delta_t$, is bounded. The probability can be determined via the CDF of the binomial distribution, which can be bounded via the Chernoff bound:
$\delta_t\leq |U| 4^{\frac{m}2}a_t^{\frac{m}2}(1-a_t)^{\frac{m}2}$.
Then the probability that the consensus labels remain constant for the remainder, i.e., the sum of $\delta_t$ from $t_0$ to $\infty$,
is bounded as well. Using the assumption that $a_t$ grows linearly, we can express this infinite series as
$\sum_{t=t_0}^\infty \delta_t\lesssim \sum_{t=0}^\infty \frac{1}{t}^{\frac{m}{2}} -\sum_{t=0}^{t_0} \frac{1}{t}^{\frac{m}{2}}\enspace ,$
that is, the difference of the Riemann zeta function and the $t_0$-th generalized harmonic number, 
$\sum_{t=t_0}^\infty \delta_t\lesssim \zeta(m/2)-H_{t_0}^{m/2}$. 
This difference can be expressed via the Hurwitz zeta function $\zeta(m/2,t_0+1)$.
\end{proof}
The full proof is provided in App.~\ref{app:convergence}\footnote{\textbf{Extended version}: https://arxiv.org/abs/2310.05696}.
%
%

%
\paragraph{Communication Complexity: }
The communication complexity of $\aimhi$ is in the same order as standard federated learning, i.e., treating the message size as a constant, the communication complexity is in $\bigo(T/b)$, where $b$ is the communication period. The actual number of bits transferred in each round, however, depends on the size of $U$: Encoding predictions as binary vectors, for a classification problem with $C\in\NN$ classes the communication complexity is in $\bigo(TC|U|/b)$. As \citet{bistritz2020distributed} observed, transferring predictions on $U$ can reduce communication substantially over transferring the weights of large neural networks. For example on FashionMINST with $|U|=10^4$ and a a neural network with $669\,706$ parameters, \fedct and \fedavg both achieve a test accuracy of $0.82$, resp. $0.83$ (cf. Tab.~\ref{table:iidexp}), but \fedct transmits only $\approx 12.2KB$ bits in each round, whereas \fedavg transmits $\approx 2.6MB$. Thus, \fedct reduces communication by a factor of $\approx 214$.
\section{Differential Privacy for \fedct}
\label{sec:privacy}
\begin{figure*}[t]
\begin{minipage}[t]{0.48\textwidth}
    \centering
    \includegraphics[width=0.9\linewidth]{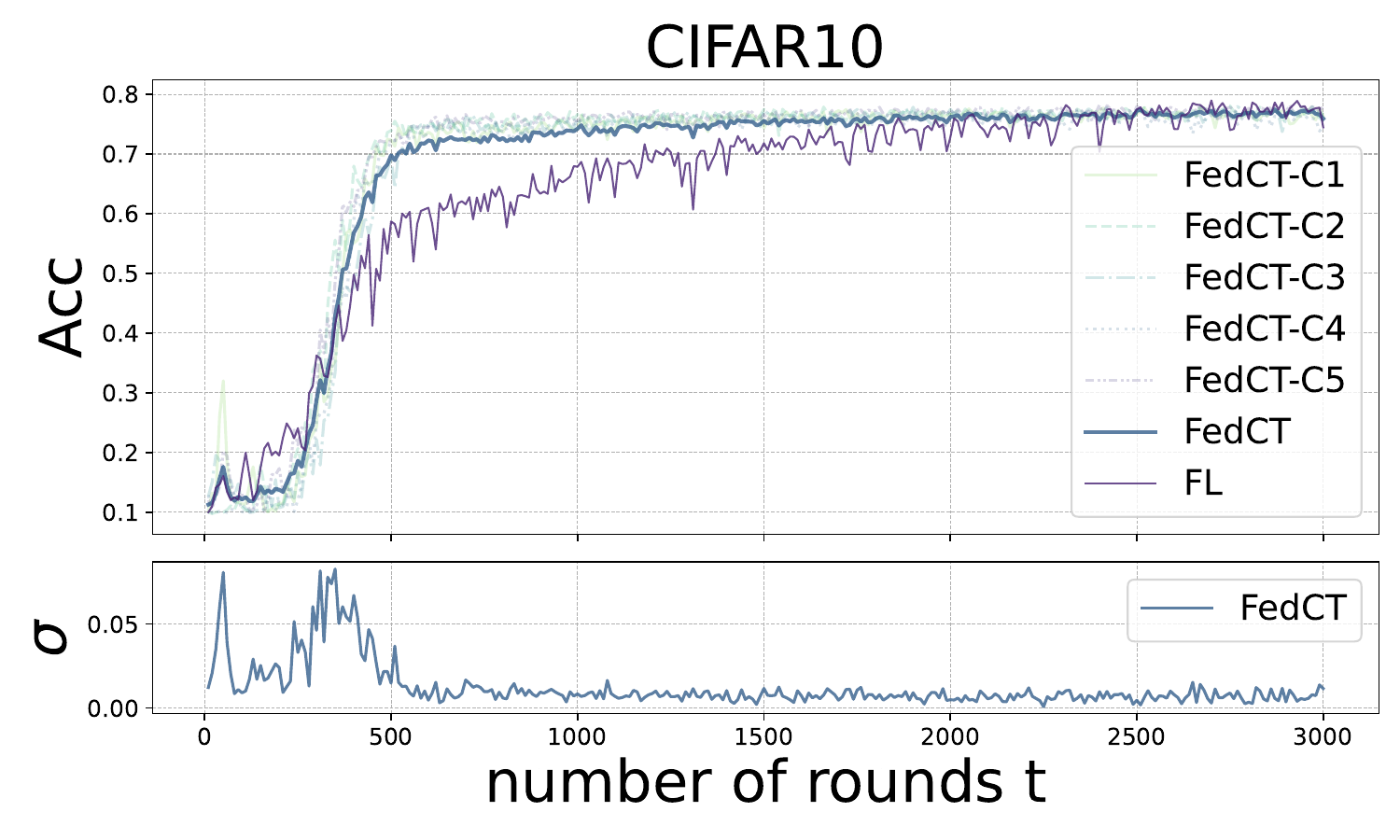}
    \caption{\textbf{Top}: Test ac. (ACC) over time on CIFAR10 of FL, and \fedct's local models and their mean. \textbf{Bottom}: Standard deviation of test accuracy of local models in \fedct.\vspace{-0.3cm}}
    \label{fig:iid:conv}
\end{minipage}\hfill
\begin{minipage}[t]{0.48\textwidth}
    \centering
    \includegraphics[width=0.9\linewidth]{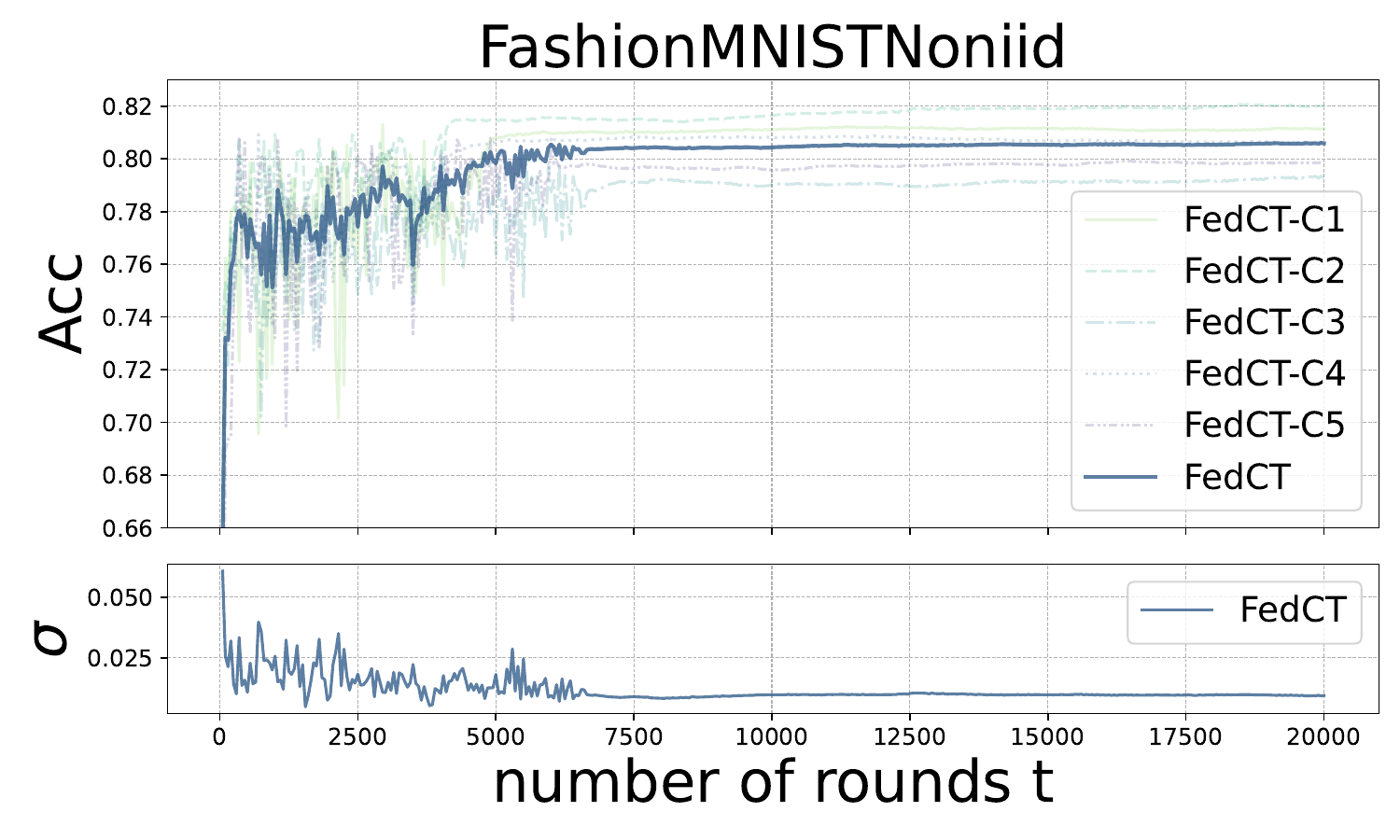}
    \caption{\textbf{Top}: Test acc- (ACC) over time of \fedct on non-iid distribution on FashionMNIST. \textbf{Bottom}: Standard deviation of test accuracy of local models in \fedct.\vspace{-0.3cm}}
    \label{fig:convnoniid}
\end{minipage}
\end{figure*}

We assume the following attack model: clients are honest and the server is honest-but-curious (or semi-honest, i.e., it follows the protocol execution correctly, but may try to infer sensitive information about clients). The attack goal is to infer sensitive information about local training data from shared information. This assumption is stronger than an attacker intercepting individual communication, or an honest-but-curious client, since the server receives shared information from all clients. We also assume that parties do not collude. Details are deferred to App.~\ref{app:privacy:details}. Sharing hard labels on an unlabeled dataset empirically improves  privacy substantially\footnote{This differs from label leakage~\citep{li2021membership}, where predictions on the private data are shared.}, as we show in Sec.~\ref{sec:experiments}. 
An empirical improvement in privacy is, however, no guarantee. Differential privacy instead provides a fundamental guarantee of privacy which is achieved through randomization of shared information: A randomized mechansim $\mathcal{M}$ with domain $\Xcal$ and range $\Ycal$ is $\epsilon$-differential private if for any two neighboring inputs $D, D^{'} \subset \Xcal $ and for a substet of outputs $S\in \Ycal $ it holds that
 $P\left(\mathcal{M}(D) \in S\right) \leq \exp(\epsilon) P\left(\mathcal{M}(D^{'}) \in S\right)$ \citep{dwork2014algorithmic}.
To obtain differential privacy (DP), the randomization has to be suitable to the information that is shared. In \fedct local clients share hard labels, i.e., categorical values in case of classification. Standard DP mechanisms, like the Gaussian~\citep{dwork2014algorithmic} or Laplacian mechanissm~\citep{dwork2006calibrating} are not suitable for categorical data. Therefore, we interpret labels as binary vectors via one-hot encoding. That is, for an unlabeled dataset $U$ and a classification problem with $C\in\NN$ classes, the predictions sent by a client with local dataset $D\subset\Xcal$ to the server can be interpreted as the binary matrix output of a deterministic mechanism $f(D)\in\{0,1\}^{|U|\times C}$.
Then, any DP-mechanism for binary data can be used, such as the XOR mechanism~\citep{ji2021differentially}, or randomized graph-coloring~\citep{d2021differential}\footnote{Label differential privacy~\citep{ghazi2021deep} is applicable, but require a different sensitivity analysis.}. 

To compute an actual DP-guarantee for a given DP-mechanism requires a bound on the sensitivity of the underlying deterministic function $f$. That is, given two neighboring datasets $D,D'$ (i.e., they differ only in a single element), the sensitivity of $f$ is defined as 
$s_f=\sup _{f(D), f\left(D^{\prime}\right)}\left\|f(D) \oplus f\left(D^{\prime}\right)\right\|_F^2\enspace ,$
where $\oplus$ denotes binary XOR.
For \fedct this means providing a bound on how much the predictions of a client on the unlabeled dataset can change if one element of its local training set is removed. In the following, we bound the sensitivity of any learning algorithm that is on-average-replace-one-stable with monotonically increasing stability rate $\epsilon:\mathbb{N}\rightarrow\mathbb{R}$ for a learning algorithm $\mathcal{A}$ and loss function $\ell$ (cf. \citet{shalev2014understanding}).
\begin{restatable}{proposition}{sensitivity}
    For classification models $h:\mathcal{X}\rightarrow\mathcal{Y}$, let $\ell$ be a loss function that upper bounds the $0-1$-loss and $\algo$ a learning algorithm that is on-average-replace-one stable with stability rate $\epsilon(n)$ for $\ell$. Let $D\cup U$ be a local training set with $|U|=n$, and $\delta\in (0,1)$. Then with probability $1-\delta$, the sensitivity $s_*$ of $\algo$ on $U$ is bounded by
    \[
        s_* \leq \left\lceil n\epsilon(n) + P\sqrt{n\epsilon(n)(1-\epsilon(n))} + \frac{P^2}{3}\right\rceil\enspace ,
    \]
    where $P=\Phi^{-1}(1-\delta)$ and $\Phi^{-1}$ is the probit function.
\label{prop:sensitivity}
\end{restatable}
The proof is provided in App.~\ref{app:sensitivity}. We now provide a DP analysis using the XOR-mechanism~\citep{ji2021differentially}. For that, let $\mathcal{B}\in\{0,1\}^{N\times P}$ denote a matrix-valued Bernoulli random variable, i.e., $\mathcal{B}\sim {\rm{Ber}}_{N,P}(\mathbf{\Theta},\mathbf{\Lambda}_{1,2},\cdots,\mathbf{\Lambda}_{N-1,N})$ with a matrix-valued Bernoulli distribution with quadratic exponential dependence structure. Here, $\mathbf{\Theta}$ is the ${P\times P}$ association parametric matrix including log-linear parameters describing the association structure of the columns, and $\mathbf{\Lambda}_{i,j}$ is the ${P\times P}$ association parametric matrix of rows $i$ and $j$. The XOR-mechanism applies this random matrix to the output of the deterministic mechanism via the XOR operator $\oplus$ and yields a randomized mechanism $\mathcal{M}(D)=f(D)\oplus \mathcal{B}$. 
We represent local predictions of clients $L^i_t$ as binary matrices and apply the XOR-mechanism, i.e., $\widehat{L}^i_t = L^i_t\oplus\mathcal{B}$. These randomized predictions are then send to the server, resulting in differentially private federated co-training (\dpaimhi): Defining the sensitivity of \dpaimhi as $s_*=\max\{s_{f^1},\dots,s_{f^m}\}$, it follows directly from Theorem~1 in~\citet{ji2021differentially} that \dpaimhi achieves $\epsilon$-differential privacy. 
\begin{corollary}
Applying the XOR mechanism to \fedct with sensitivity $s_*$ achieves $\epsilon$-DP if $\mathbf{\Theta}$ and $\mathbf{\Lambda}_{i,j}$ satisfy
\begin{equation}
s_*(||\boldsymbol{\lambda}(\mathbf{\Theta})||_2    +   \textstyle\sum_{i=1}^{N-1}\sum_{j= i+1}^N  ||\boldsymbol{\lambda}(\mathbf{\Lambda}_{i,j})||_2  \Big)\leq \epsilon,  
\label{condition_general}
\end{equation}
where 
$||\boldsymbol{\lambda}(\mathbf{\Theta})||_2$ and $||\boldsymbol{\lambda}(\mathbf{\Lambda}_{i,j})||_2$ are the $l_2$ norms of the eigenvalues of $\mathbf{\Theta}$ and $\mathbf{\Lambda}_{i,j}$. 
\label{cor:dpaimhi}
\end{corollary}

%
On-average-replace-one-stability holds for many supervised learning methods. For example, every regularized risk minimizer for a convex, Lipschitz loss using a strongly convex regularizer, like Thikonov-regularization, is on-average-replace-one-stable~\citep[cf. Chp. 13.3 in ][]{shalev2014understanding}, as well as deep learning with SGD~\citep{hardt2016train}, including cases involving non-smooth loss functions~\citep{bassily2020stability}. We empirically evaluate the privacy-utility trade-off of \fedct with differential privacy in Sec.~\ref{sec:experiments}.

\section{Empirical Evaluation}
\label{sec:experiments}

We empirically show that \fedct presents a favorable privacy-utility trade-off compared to other federated learning approaches by showing that it achieves similar test accuracy with substantially improved privacy. We compare \fedct to federated averaging~\citep{mcmahan2017communication} (\fedavg), differentially private federated averaging (\dpfl) achieved via the Gaussian mechanism \citep{geyer2017differentially}, and distributed distillation~\citep{bistritz2020distributed} (\dd)\footnote{The code to reproduce all experiments is available at \oururl} on $3$ benchmark datasets and $2$ medical image classification datasets, as well as on a fine-tuning task for large language models. We also compare with \pate~\citep{papernot2016semi}, although it is not collaborative, because it shares hard labels (see App.~\ref{app:PATE} for details).

\paragraph{Experimental Setup}
We use three benchmark image classification datasets, CIFAR10~\citep{krizhevsky2010cifar}, FashionMNIST~\citep{xiao2017/online}, and SVHN~ \citep{netzer2011reading}, as well as two real medical image classification datasets, MRI scans for brain tumor detection~\citep{MRIdata}, 
and chest X-rays for pneumonia detection~\citep{kermany2018identifying}. 
We evaluate \fedct with interpretable models on five benchmark datasets, WineQuality~\citep{cortez2009modeling}, Breastcancer~\citep{street1993nuclear}, AdultsIncome~\citep{misc_adult_2}, Mushroom~\citep{misc_mushroom_73}, and Covertype~\citep{misc_covertype_31}. We fine-tune an LLM on the IMDB dataset~\citep{maas-EtAl:2011:ACL-HLT2011} and the Twitter dataset~\citep{TwitterSentimentAnalysis}. We first divide each dataset into a test and training set and further divide the training set into an unlabeled dataset $U$ and a set of $m$ local training sets (sampling iid. for all experiments, except for the experiments on heterogeneous data distributions). We also investigated how the distribution and size of unlabeled datasets affect performance, as demonstrated in experiments App.\ref{app:unlabeled:distribution} and App.\ref{app:effect:unlabeled}. The architectures of the neural networks are provided in App.~\ref{app:exp:details}. The hyper-parameters are optimized individually for all methods on a subset of the training set via cross-validation. We select the number of rounds to be the maximum rounds required so that all methods converge, i.e., $T=2*10^4$.
We measure empirical privacy vulnerability by performing a large number of membership inference attacks and compute the probability of inferring upon sensitive data, using the ML Privacy Meter tool~\citep{murakonda2020ml}. The \textbf{vulnerability (VUL)} of a method is the ROC AUC of membership attacks over $K$ runs over the entire training set. A vulnerability of $1.0$ means that membership can be inferred with certainty, whereas $0.5$ means that deciding on membership is a random guess. While our privacy evaluation reflects the relative performance of the different methods, it can be inaccurate in assessing the actual privacy risk, particularly for the most vulnerable data points~\citep{aerni2024evaluations}. More details on the follwoing experiments, additional experiments on using mixed model types, a comparison with \fedmd~\citep{li2019fedmd}, and an abblation study can be found in App.~\ref{app:extraexp}

\begin{table*}[bt]
\centering
\caption{Test acc. (ACC) and privacy vulnerability (VUL, smaller is better) for $m=5$ clients on iid data.}
\begin{adjustbox}{width=1\textwidth}
\small
\begin{tabular}{c|cc|cc|cc|cc|cc}
\hline 
\textbf{Method} & \multicolumn{2}{c|}{\textbf{CIFAR10}} & \multicolumn{2}{c|}{\textbf{FashionMNIST}} & \multicolumn{2}{c|}{\textbf{Pneumonia}} & \multicolumn{2}{c|}{\textbf{MRI}} & \multicolumn{2}{c}{\textbf{SVHN}} \\
 & ACC & VUL & ACC & VUL & ACC & VUL & ACC & VUL & ACC & VUL \\
\hline
\textbf{\fedct} & $\mathbf{0.77}\pm 0.003$ & $0.52$ & $\mathbf{0.84}\pm 0.004$ & $\mathbf{0.51}$ & $\mathbf{0.78}\pm 0.008$ & $\mathbf{0.51}$ & $0.64\pm 0.004$ & $0.52$ & $\mathbf{0.91}\pm 0.002$ & $\mathbf{0.53}$ \\
\textbf{\dpaimhi} $(\epsilon=0.1)$ & $0.76 \pm 0.002$ & $\mathbf{0.51}$ & $0.80 \pm 0.001$ & $0.52$ & $0.75 \pm 0.004$ & $\mathbf{0.51}$ & $0.62 \pm 0.002$ & $\mathbf{0.51}$ & $0.86\pm 0.001$ & $\mathbf{0.53}$ \\
\textbf{\fedavg} & $\mathbf{0.77}\pm 0.020$ & $0.73$ & $0.83\pm 0.024$ & $0.72$ & $0.74\pm 0.013$ & $0.76$ & $0.66\pm 0.015$ & $0.73$ & $\mathbf{0.91}\pm 0.026$ & $0.71$ \\
\textbf{\dpfl} $(\epsilon=0.1)$ & $0.68 \pm 0.002$ & $0.70$ & $0.69 \pm 0.002$ & $0.71$ & $0.61\pm 0.004$ & $0.69$ & $0.56\pm 0.003$ & $0.62$ & $0.71\pm 0.005$ & $0.70$ \\
\textbf{\dd} & $0.70\pm 0.012$ & $0.61$ & $0.82 \pm 0.016$ & $0.60$ & $\mathbf{0.78} \pm 0.003$ & $0.63$ & $\mathbf{0.68} \pm 0.008$ & $0.60$ & $0.73 \pm 0.014$ & $0.59$ \\
\textbf{\pate} & $0.69\pm 0.002$ & $0.60$ & $0.73\pm 0.001$ & $0.59$ & $0.75\pm 0.003$ & $0.59$ & $0.61\pm 0.001$ & $0.60$ & $0.87\pm 0.002$ & $0.58$ \\
\textbf{\dppate} & $0.67 \pm 0.003$ & $0.58$ & $0.73 \pm 0.002$ & $0.57$ & $0.71\pm 0.001$ & $0.58$ & $0.60 \pm 0.001$ & $0.57$ & $0.86 \pm 0.002$ & $0.57$ \\
\hline
\end{tabular}
\end{adjustbox}
\vspace{0.1cm}
\vspace{-0.4cm}
\label{table:iidexp}
\end{table*}

\begin{figure}[ht]
\vspace{-0.2cm}
    \centering
    \includegraphics[width=0.95\linewidth]{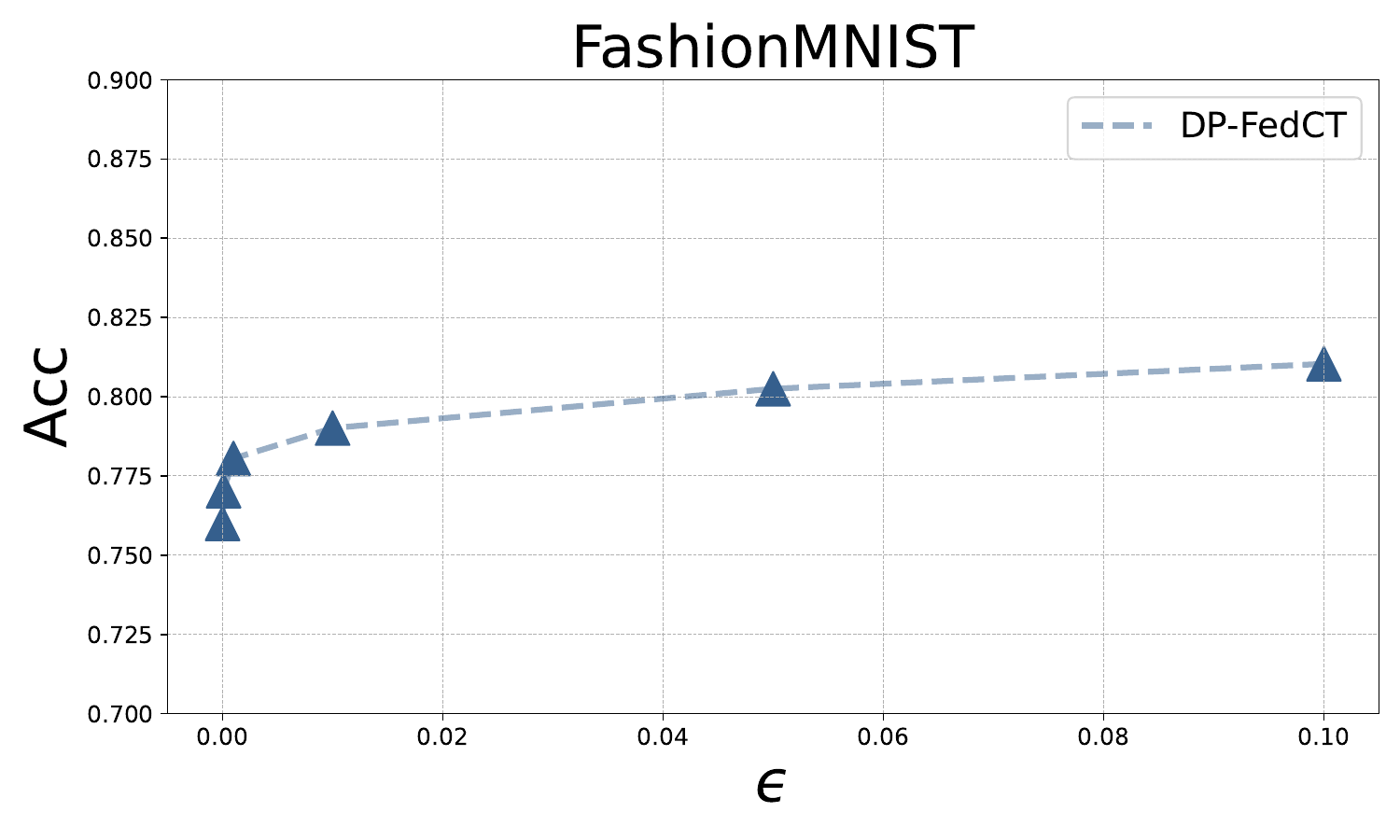}
    \caption{Accuracy (ACC) of \dpaimhi on the FashionMNIST dataset under different levels of privacy $\epsilon$.\vspace{-0.3cm}}
    \label{fig:dpaimhi}
\end{figure}
\paragraph{Privacy-Utility-Trade-Off:}
We first evaluate the performance of \fedct and baselines for deep learning on homogeneous data distributions. 
We use an unlabeled dataset of size $|U|=10^4$ for CIFAR10, $|U| = 5\cdot 10^4$ for FashionMNIST, $|U|=170$ for MRI, $|U|=900$ for Pneumonia, and $|U|=35\cdot 10^4$ for SVHM. Note that only \fedct, \dpfedct, \pate, and \dd use the unlabeled dataset. The remaining training data is distributed over $m=5$ clients. We repeat all experiments $3$ times and report average test accuracy and maximum deviation (see App.~\ref{app:exp:details} for details). 

The results presented in Tab.~\ref{table:iidexp} show that \fedct achieves a test accuracy comparable to both \fedavg, and \dd, while preserving privacy to the highest level. That is, \fedct performs best on CIFAR10, has a similar performance to both on FashionMNIST, Pneumonia, and SVHN, and 
is slightly worse on MRI. All collaborative approaches outperform \pate. The vulnerability is around $0.5$, so that membership inference attacks are akin to random guessing. \fedavg instead consistently has a vulnerability over $0.7$. \dpfl improves privacy, but reduces the test accuracy substantially. 
Our experiments show that \dd substantially improves privacy over both \fedavg and \dpfl, yet still is vulnerable ($VUL\approx 0.6$). 
Since \fedct does not produce a global model, we investigate the convergence behavior of individual client models in terms of test accuracy on CIFAR10 in Fig.~\ref{fig:iid:conv}. From the standard deviation between clients $\sigma$ we see that they converge to a consensus after around $700$ rounds with only slight deviations afterwards. Overall, \fedct converges slightly faster than \fedavg, though the latter increases its test accuracy slightly further, eventually.

\begin{table}
\centering
    \caption{Non-iid data distributions with $m=10$ clients.} 
    \begin{adjustbox}{width=0.45\textwidth}
    \small
    \centering
    \begin{tabular}{l|c|c|c}
        \hline
        & \vtop{\hbox{\strut $\alpha_1=100$}\hbox{\strut $\alpha_2=2$}} & \vtop{\hbox{\strut $\alpha_1=100$}\hbox{\strut $\alpha_2=0.01$}}& \vtop{\hbox{\strut $\alpha_1=0.01$}\hbox{\strut $\alpha_2=0.01$}} \\
        \hline
        \fedct & $0.7975$ & $0.7905$ & $0.6252$ \\
        \fedct(QM) & $0.7968$ & $0.7808$ & $0.7358$ \\
        \fedavg & $0.8150$ & $0.7950$ & $0.7346$ \\
        \dd & $0.7684$ & $0.7219$ & $0.5580$ \\
        \hline
    \end{tabular}
    \end{adjustbox}
    \vspace{-0.2cm}   
    \label{tab:alpha} 
\end{table}
\paragraph{Privacy-Utility Trade-Off With Differential Privacy:}
Differential privacy guarantees typically come at a cost in terms of utility, which in our case means a loss in model quality. Analyzing this privacy-utility trade-off requires estimating the sensitivity. Since stability-bounds for neural networks tend to underestimate the on-average-replace-one stability, leading to vacuous results for generalization~\citep{nagarajan2019uniform, petzka2021relative}, using them to bound sensitivity would underestimate utility. 
Using an empirical approximation provides a more accurate estimate for the privacy-utility trade-off~\citep{rubinstein2017pain}. 
To get this approximation, we apply \fedct with $m=5$ clients on the FashionMNIST dataset~\citep{xiao2017/online} for various privacy levels $\epsilon$. We estimate the sensitivity of \dpaimhi by sampling $n=100$ datasets $D'_1,\dots,D'_n$ neighboring a local training set $D$ to approximate 
$s_*\approx \max_{i\in [n]}\|f(D)\oplus f(D'_i)\|_F^2\enspace ,$
which yields $s_*\approx3000$. 
Using this estimate, Fig.~\ref{fig:dpaimhi} shows that \dpaimhi achieves a high utility in terms of test accuracy even for moderate-to-high privacy levels $\epsilon$ with an accuracy of $0.8$ for $\epsilon=0.1$\footnote{Note that using the trivial upper bound of $s^W_*=|U|=5\cdot 10^4$ instead of our estimate results in a slightly higher epsilon: for a noise level that achieves $\epsilon=0.1$ with the empirical estimate of $s_*$, the worst-case bound results in $\epsilon = 0.1\cdot s_*^W/s_*=5/3$, instead.} (without any noise, \fedct achieves an accuracy of $0.82$ in this setup). As~\citet{papernot2016semi} observed, the reason for the good trade-off probably lies in the consensus mechanism: for a single unlabeled example $\mu>m/C$ clients predict the majority class, so the XOR-mechanism has to change the predictions of at least $\mu-m/C$ many clients to change its consensus label. 
%

%

\begin{table*}[t]
\centering 
\caption{Test accuracy (ACC) of \fedct with interpretable models.}
\begin{adjustbox}{width=1\textwidth}
\small
\begin{tabular}{c|cc|cc|cc|cc}
\hline 
\multicolumn{1}{c|}{Dataset}& \multicolumn{2}{c|}{DT}&\multicolumn{2}{c|}{RuleFit} &\multicolumn{2}{c|}{XGBoost} &\multicolumn{2}{c}{Random Forest} \\
\multicolumn{1}{c|} {}&\fedct&\cent&\fedct&\cent&\fedct&\cent&\fedct&\cent\\
\hline
 WineQuality &$0.95 \pm 0.01$&$0.92$& $0.93 \pm 0.01 $&$0.95$& $0.94 \pm 0.01$&$0.94$& $0.96 \pm 0.01$&$0.98$ \\
 BreastCancer&$0.89 \pm 0.01 $&$0.89$& $0.92 \pm 0.01$&$0.93$&$0.93 \pm 0.01$&$0.94$& $0.90 \pm 0.02$&$0.93$  \\
 AdultsIncome&$0.81 \pm 0.01 $&$0.82$& $0.84 \pm 0.02$&$0.85$&$0.85 \pm 0.02$&$0.87$& $0.85 \pm 0.01$&$0.86$  \\
 Mushroom    &$0.98 \pm 0.01 $&$1$& $0.98 \pm 0.02$&$1$&$0.98 \pm 0.01$&$1$& $0.99 \pm 0.01$&$1$  \\
 Covertype   &$0.88 \pm 0.02$&$0.94$& $0.73 \pm 0.02$&$0.76$&$0.84 \pm 0.02$&$0.87$& $0.90 \pm 0.01 $&$0.95$  \\
\hline
\end{tabular}
\end{adjustbox}
\vspace{-0.4cm}
\label{table:interpmodels}
\end{table*}
\paragraph{Heterogeneous Data Distributions:}
In many realistic applications, local datasets are not heterogeneous. We show that \fedct performs similar to \fedavg for non-pathological non-iid data distributions, but \fedct with majority voting, as well as soft label sharing, are outperformed on pathological non-iid distributions. This is remedied by using a qualified majority vote as consensus. 
For a non-pathological non-iid data distribution, we sample half of the training data from a Dirichlet distribution over labels with $\alpha_1=100$ (mild heterogeneity) and the other half with $\alpha_2=2$ (medium heterogeneity) and $\alpha_2=0.01$ (strong heterogeneity). For both cases, we see that \fedct, \dd, and \fedavg perform similarly (see Tab.~\ref{tab:alpha}). In Fig.~\ref{fig:convnoniid} we show the convergence behavior of individual clients for $\alpha_1=100, \alpha_2=2$ which is similar to the iid case, but with higher variance between individual client models. For the pathological case ($\alpha_1=\alpha_2=0.01$) \fedavg still performs well, outperforming both \fedct and \dd. We conjecture that a meaningful consensus requires clients to achieve a minimum performance for all labels. In the pathological case, clients observe only a small subset of labels and thus perform poorly on a majority of data. Using a qualified majority (\fedct(QM)) with a quorum of $0.9$ improves the performance of \fedct to the level of \fedavg, showing that \fedct using QM also performs well in the pathological case. Further details are deferred to App.~\ref{app:deepnoniid}.


%
\paragraph{Effect of Unlabeled Dataset Distribution}
\label{app:unlabeled:distribution}
In our main experiments in Tab.\ref{table:iidexp}, the unlabeled set $U$ is drawn iid from the respective datasets. In practice, however, publicly available datasets will be similar, but not equal to a private dataset's distribution. We therefore explore what the impact of the distribution of the unlabeled dataset is on the performance of \fedct. We investigate this using iid samples from the CIFAR10 dataset as private data. We then test \fedct on various unlabeled distribution, ranging from similar to private data to very dissimilar. As very similar unlabeled dataset, we use a non-iid sample of CIFAR10 sampling using a Dirichlet distribution over labels with $\alpha=0.5$. As dissimilar unlabeled dataset, we use an iid sample of CIFAR100---the distribution is highly dissimilar, since not a single class of CIFAR10 is present in CIFAR100. As a middle ground, we have chosen classes from CIFAR100 that are semantically more similar to the CIFAR10 classes (see Tab.~\ref{table:classmappingC100C10}). In all cases we have chosen $|U|=8500$. The results shown in Tab.~\ref{table:difdist} show that while the test accuracy of \fedct decreases the more dissimilar the unlabaled data distribution is from private data, it remains high even for very dissimilar distributions.

\begin{table}[H]
\hfill
\centering
    \caption{Test accuracy (ACC) of \fedct on CIFAR10 for three different scenarios of using a public dataset with different data distribution.}
    \begin{adjustbox}{width=0.35\textwidth}
    \small
    \begin{tabular}{c|c}
        \hline 
        \textbf{Distribution of }$|U|$ &\textbf{\fedct} \\
        \hline
        iid CIFAR10 & $0.77$ \\
        non-iid CIFAR10 ($\alpha=1.0$) & $0.74$ \\
        non-iid CIFAR10 ($\alpha=0.5$) & $0.72$ \\
        similar CIFAR100 classes & $0.71$ \\
        iid sample of CIFAR100  & $0.65$ \\
        \hline
    \end{tabular}
    \end{adjustbox}
    \label{table:difdist}
\end{table}

\paragraph{Interpretable Models:}
\fedct allows training models that cannot be aggregated in \fedavg and cannot be trained via soft label sharing (e.g., as in \dd). Many interpretable models, such as decision trees~\citep{quinlan1986induction}, XGBoost~\citep{chen2016xgboost}, Random Forest~\citep{breiman2001random}, and RuleFit~\citep{friedman2008predictive} fall under this class. We evaluate \fedct on such models on $5$ benchmark datasets with $m=5$ clients and compare its performance to pooling all data and training a model centrally (Centralized). 
The results in Tab.~\ref{table:interpmodels} show that \fedct can train interpretable models in a federated learning setup, achieving a model quality comparable to centralized training. In App.~\ref{app:mix:model} we show that \fedct can also train mixed models, i.e., each client training a different model type, to high accuracy.

\paragraph{Fine-Tuning Large Language Models:}
In fine-tuning, model quality is already high from the start so that pseudo-labels are likely of high quality from the start. If true, semi-supervised approaches should improve performance over \fedavg. To test this hypothesis, we fine-tune the GPT2 model transformer with a sequence classification head (linear layer) comprising of $124.44$ million parameters on the IMBD sentimental dataset~\citep{maas-EtAl:2011:ACL-HLT2011} and the Twitter sentiment dataset ~\citep{TwitterSentimentAnalysis}, using $m=10$ clients with $|U|=150$ for IMBD and $|U|=35,000$ for Twitter. Indeed, we observe that on IMDB, \fedct achieves a test accuracy of $0.73$, whereas \fedavg achieves an ACC of $0.59$ and on Twitter \fedct achieves a test accuracy of $0.65$, whereas \fedavg achieves an ACC of $0.61$ (see App.~\ref{app:llm} for details). 

\section{Discussion and Conclusion}
\label{sec:discussion}
We propose a federated semi-supervised co-training approach that collaboratively trains models via sharing predictions on an unlabeled dataset $U$. In many applications, such unlabeled datasets are available, e.g. in healthcare\footnote{Examples for large public health databases are the US NCHS DB,  UK NHS DB, UK Biobank~\citep{sudlow2015uk}, MIMIC-III database~\citep{johnson2016mimic}, TCGA public dataset~\citep{tcga}, and EU EHDS.}, or can be synthetically generated~\citep{el2020practical}. 
A limitation of \fedct is that it does not produce a global model. Instead, it promotes agreement between local models. Our experiments, however, show that local models quickly converge to similar test accuracy so that each local model could act as global model. At the same time, \fedct allows us to use different models that can be tailored to each client (see App.~\ref{app:mix:model}). A second limitation, revealed by our experiments, is that on pathological non-iid data, where clients only observe a small subset of labels, soft label sharing and hard label sharing with majority voting are outperformed by \fedavg. Using a qualified majority vote as consensus remedies this issue. Similar to federated learning variants tailored to heterogeneous data, such as FedProx~\citep{li2020federated} and SCAFFOLD~\citep{karimireddy2020scaffold}, it would make for excellent future work to further improve \fedct's performance in this case, e.g., by using more elaborate consensus mechanisms~\citep{warfield2004simultaneous}.
Furthermore, investigating client subsampling in \fedct and its impact on the consensus mechanism, other communication-efficient strategies~\citep[e.g.,][]{kamp2016communication, kamp2019efficient}, and learning from small datasets~\citep{kamp2023federated} is interesting. 
The results on fine-tuning LLMs are promising and suggest that semi-supervised learning can be particularly beneficial in federated fine-tuning of foundation models, which will be interesting to further investigate in the future.

We show that \fedct matches the model quality of \fedavg and \dd while significantly improving privacy over both \fedavg and \dd, as well as \dpfl. From this we conclude that sharing little is enough: sharing hard labels improves privacy substantially while maintaining a favorable privacy-utility trade-off, in particular for fine-tuning LLMs. Moreover, \fedct allows us to train interpretable models, such as rule ensembles and XGBoost, in a federated learning setup. The proposed approach facilitates the deployment of machine learning in critical domains such as healthcare, where highly sensitive private datasets are distributed across sites, and public unlabeled datasets are often available.

\vfill
\pagebreak

\section*{Acknowledgements}
Amr Abourayya, Jens Kleesiek, and Michael Kamp received support from the Cancer Research Center Cologne Essen (CCCE). Erman Ayday was partly supported by the National Science Foundation (NSF) under grant numbers 2141622, 2427505, and OAC-2112606.
\bibliography{references}

\newpage
\appendix
\onecolumn

\makeatletter
\@addtoreset{subsection}{section} 
\renewcommand{\thesubsection}{\thesection.\arabic{subsection}} 
\makeatother


\section*{Supplementary Material}
In this Appendix, we provide supplementary information  
\begin{itemize}
    \item Proof of convergence ( Proposition 1 ) of \fedct in App.~\ref{app:convergence},
    \item Proof for Proposition 2, which addresses the sensitivity analysis of \fedct in App.~\ref{app:sensitivity},
    \item Discussion of Additional Related work in App \ref{app:smifed:discussion},
    \item Additional Empirical Evaluation in App.~\ref{app:extraexp},
    \begin{itemize}
    \item Mixed Model Types: Evaluation of \fedct with different types of models (e.g., neural networks, decision trees) across clients.
    \item Fine-Tuning Large Language Models (GPT-2): Experimental results on fine-tuning large language models in a federated setting, highlighting the model's performance and generalization abilities.
    \item Detailed Comparison to PATE: A thorough comparison with the Private Aggregation of Teacher Ensembles (PATE) framework.
    \item Comparison to FedMD: Empirical results comparing our method with the FedMD approach, emphasizing the advantages in heterogeneous data scenarios.
    \item Scalability Results
    \item Detailed Analysis of Heterogeneous Data Distributions: Exploration of the algorithm's robustness and performance under varying degrees of data heterogeneity among clients.
    \item Ablation Study
    \item A Note on the Byzantine Resilience of FEDCT.
\end{itemize}
    \item Details on the experimental setup in App.~\ref{app:exp:details},
    \item Practical impact of Federated Co-training approach in App \ref{sec:Impact}
\end{itemize}

\section{Proof of Proposition~\ref{prop:convergence}}
\label{app:convergence}
For convenience, we restate the proposition.
\convergence*
\begin{proof}
Let $P_t$ denote the consensus label at time $t\in\NN$. We first show that the probability $\delta_{t}$ of $P_{t}\neq P_{t-1}$ is bounded. Since the learning algorithm $\algo$ at time $t\geq t_0$ achieves a training accuracy $a_t\geq 0.5$, the probability can be determined via the CDF of the binomial distribution, i.e.,
\begin{equation*}
    \begin{split}
        \delta_t=&\Prob{}{\exists u\in U:\sum_{i=1}^m\mathds{1}_{h^i_t(u) = v}<\left\lfloor\frac{m}{2}\right\rfloor}\\
        =&F\left(\left\lfloor\frac{m}{2}\right\rfloor-1,m,a_t\right)=\sum_{i=1}^{\left\lfloor\frac{m}{2}\right\rfloor-1}{\binom{m}{i}}a_t^i(1-a_t)^{m-i}\enspace ,
    \end{split}
\end{equation*}
where $v$ is the consensus label at $t-1$. Note that this is a worst-case bound, since for $C>2$ a majority can be achieved already with $\lfloor m/C\rfloor$ many votes. However, for any number of classes it is guaranteed that $m/2+1$ votes is the majority. Applying the Chernoff bound and denoting by $D(\cdot\|\cdot)$ the Kullback-Leibler divergence yields
\begin{equation*}
    \begin{split}
        \delta_t \leq& \exp\left(-mD\left(\frac{\left\lfloor\frac{m}{2}\right\rfloor-1}{m}\: \middle\|\: a_t\right)^2\right)\\
        =&\exp\left(-m\left(\frac{\left\lfloor\frac{m}{2}\right\rfloor-1}{m}\log\frac{\frac{\left\lfloor\frac{m}{2}\right\rfloor-1}{m}}{a_t}+\left(1-\frac{\left\lfloor\frac{m}{2}\right\rfloor-1}{m}\right)\log\frac{1-\frac{\left\lfloor\frac{m}{2}\right\rfloor-1}{m}}{1-a_t}\right)\right)\\
        \leq&\exp\left(-m\left(\frac{\frac{m}{2}}{m}\log\frac{\frac{\frac{m}{2}}{m}}{a_t}+\left(1-\frac{\frac{m}{2}}{m}\right)\log\frac{1-\frac{\frac{m}{2}}{m}}{1-a_t}\right)\right)\\
        =&\exp\left(-m\left(\frac12\log\frac{\frac12}{a_t}+\frac12\log\frac{\frac12}{1-a_t}\right)\right)=\exp\left(-\frac{m}2\log\frac{1}{2a_t}-\frac{m}2\log\frac{1}{2(1-a_t)}\right)\\
        =&\exp\left(\frac{m}2\left(\log 2a_t + \log 2(1-a_t)\right)\right)=\left(2a_t\right)^{\frac{m}2}\left(2(1-a_t\right)^{\frac{m}2}=4^{\frac{m}2}a_t^{\frac{m}2}(1-a_t)^{\frac{m}2}\enspace .
    \end{split}
\end{equation*}
The union bound over all $u\in U$ yields 
\[
\delta_t\leq |U|4^{\frac{m}2}a_t^{\frac{m}2}(1-a_t)^{\frac{m}2}\enspace .
\]
To show convergence, we need to show that for $t_0\in\NN$ it holds that 
\[
\sum_{t=t_0}^\infty \delta_t \leq \delta
\]
for $0\leq \delta < 1$. 
Since we assume that $a_t$ grows linearly, we can write wlog. $a_t=1-c/t$ for some $c\in\RR_+$ and $t\geq 2c$. With this, the sum can be written as 
\begin{equation*}
    \begin{split}
        \sum_{t=t_0}^\infty\delta_t \leq& |U|\sum_{t=t_0}^\infty 4^{\frac{m}2}\left(1-\frac{c}{t}\right)^{\frac{m}2}\left(\frac{c}{t}\right)^{\frac{m}2}=|U|4^{\frac{m}2}\sum_{t=t_0}^\infty \left(\frac{\frac{t}{c}-1}{\frac{t^2}{c^2}}\right)^{\frac{m}2}\\
        \leq&|U|4^{\frac{m}2}\sum_{t=t_0}^\infty \left(\frac{\frac{t}{c}}{\frac{t^2}{c^2}}\right)^{\frac{m}2}=(4c)^{\frac{m}2}\sum_{t=t_0}^\infty \left(\frac{1}{t}\right)^{\frac{m}2}=|U|(4c)^{\frac{m}2}\zeta\left(\frac{m}{2}\right)-H_{t_0}^{\left(\frac{m}{2}\right)}\enspace ,
        \enspace ,
    \end{split}
\end{equation*}
where $\zeta(x)$ is the Riemann zeta function and $H_n^{(x)}$ is the generalized harmonic number. Note that $H_n^{(x)}=\zeta(x)-\zeta(x,n+1)$, where $\zeta(x,q)$ is the Hurwitz zeta function, so that this expression can be simplified to
\[
\sum_{t=t_0}^\infty\delta_t \leq |U|(4c)^{\frac{m}2}\zeta\left(\frac{m}{2}\right)-\zeta\left(\frac{m}{2}\right)+\zeta\left(\frac{m}{2},t_0+1\right)=|U|(4c)^{\frac{m}2}\zeta\left(\frac{m}{2},t_0+1\right)\enspace .
\]
\end{proof}

Note that $\delta\rightarrow 0$ for $t_0\rightarrow \infty$, and $\delta$ is monotonically decreasing with $m$. To illustrate this, we plotted $\delta$ wrt. $t_0$ in Fig.~\ref{fig:convergence:illustration}: for moderate numbers of clients 
($m\geq 50$) we obtain convergence with probability $\approx1.0$ at $t_0=1000$ (for $m=50$ and $m=100$ with $c\in\{1,2,10\}$). For cross-silo scenarios ($m=5$) it depends on how fast the local accuracy increases: $\delta=0.9962$ for $c=1$, but $\delta=0.7868$ for $c=10$.

\begin{figure}[h!]
\centering
    \includegraphics[width=0.55\linewidth]{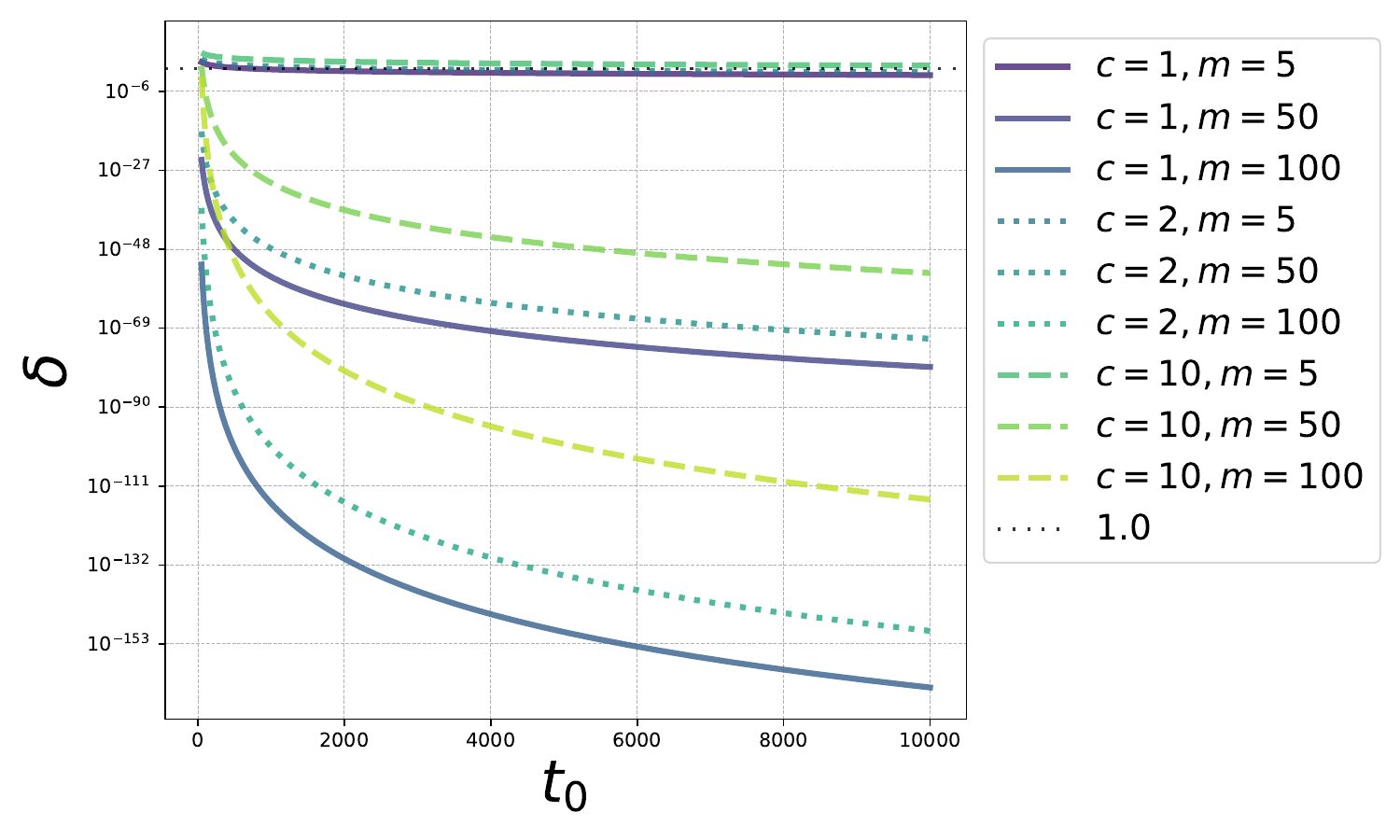}
    \caption{Numerical evaluation of Prop.~\ref{prop:convergence} for $|U|=10^4$.}
    \vspace{-0.5cm}
    \label{fig:convergence:illustration}
\end{figure}

\section{Proof of Proposition~\ref{prop:sensitivity}}
\label{app:sensitivity}
For convenience, we first recall the definition of on-average-replace-one stable. For that, we use the following notation. For a set $S=\{z_1,\dots,z_n\}$ and an additional element $z'$ we denote by $S^{(i)}$ the set where the $i$-th element has been replaced by $z'$, i.e., $S^{(i)}=\{z_1,\dots,z_{i-1},z',z_{i+1},\dots,z_n\}$.
\begin{definition}[\citep{shalev2014understanding}] (On-Average-Replace-One-Stable) Let $\epsilon: \mathbb{N} \rightarrow \mathbb{R}$ be monotonically decreasing, and $\ell$ a loss function, then a learning algorithm $\algo$ is on-average-replace-one-stable with rate $\epsilon(n)$ if for every distribution $\mathcal{D}$
\[
\underset{\stackrel{\left(S, z^{\prime}\right) \sim \mathcal{D}^{n+1}}{i \sim U(n)}}{\mathbb{E}}\left[\ell\left(\algo\left(S^{(i)}\right), z_i\right)-\ell\left(\algo(S), z_i\right)\right] \leq \epsilon(n)\enspace .
\]
\end{definition}
We furthermore restate the proposition.
\sensitivity*

\begin{proof}
    The sensitivity $s_*$ is defined as the supremum of the Frobenius norm of the symmetric difference between the predictions on the unlabeled dataset $U$ for two models $h_s$ and $h_s^{\prime}$ trained on datasets $s$ and $s^{\prime}$ that differ by one instance.
    \[
    s_*=\sup _{S, S^{\prime}}\left\|h_S(U) \Delta h_{S^{\prime}}(U)\right\|_F
    \]
    Since $\algo$ is on-average-replace-one stable with rate $\epsilon$ for $\ell$ and $\ell$ upper bounds the $0-1$-loss, $\algo$ is on-average-replace-one stable with rate at most $\epsilon$ for the $0-1$-loss. Thus, the expected change in loss on a single element of the training set is bounded by $\epsilon(|D\cup U|)$. Since the $0-1$-loss is either $0$ or $1$, this can be interpreted as a success probability in a Bernoulli process. The expected number of differences on the unlabeled dataset then is the expected value of the corresponding binomial distribution, i.e., $|U|\epsilon(|D\cup U|)\leq |U|\epsilon(|U|)$. We are interested in the maximum number of successes such that the cumulative distribution function of the binomial distribution is smaller than $1-\delta$. 
    This threshold $k$ can be found using the quantile function (inverse CDF) for which, however, no closed form exists. \citet{short2023binomial} has shown that the quantile function $Q(n,p,R)$ can be bounded by
    \[
        Q(n,p,R) \leq \left\lceil np + \Phi^{-1}(R)\sqrt{np(1-p)} + \frac{\Phi^{-1}(R)^2}{3}\right\rceil \enspace ,
    \]
    where $\Phi^{-1}$ is the probit function (inverse of standard normal's cdf).
    With $n=|u|$, $p=\epsilon(|U|)$, and $R=1-\delta$, the number of differences in predictions on the unlabeled dataset, i.e., the sensitivity $s_*$, is upper bounded by
    \[
        s_* \leq \left\lceil |U|\epsilon(|U|) + \Phi^{-1}(1-\delta)\sqrt{|U|\epsilon(|U|)(1-\epsilon(|U|))} + \frac{\Phi^{-1}(1-\delta)^2}{3}\right\rceil
    \]
    with probability $1-\delta$.
\end{proof}

 Note that we employed the relaxed version of on-average-replace-one-stability as outlined by \cite{shalev2014understanding}. Notably, a learning problem is learnable if and only if a stable learning rule exists \cite{shalev2010learnability}.While proving on-average-replace-one stability for deep learning remains a significant challenge, it has been successfully established for SGD \cite{hardt2016train}, including cases involving non-smooth loss functions \cite{bassily2020stability}.

\section{Additional Related Work}
\label{app:smifed:discussion} 
\begin{table}[ht]
    \centering
    \caption{Comparison of Different Federated Learning Methods}
    \begin{adjustbox}{width=1\textwidth}
    \begin{tabular}{l|c|c|c|c|c}
        \hline
        Method & Shared Information & Allow Heterogeneity & Public Data & Train Non-Gradient Methods & Collaborative Training \\
        \hline
        FedAvg \cite{mcmahan2017communication} & model Parameters & no & none & no& yes \\
        FedHKD \cite{chen2023best} & model parameters, representations, soft labels & no & none & no &yes\\
        FedMD \cite{li2019fedmd} & soft labels & yes & labeled & no&yes \\
        SemiFL \cite{diao2022semifl} & model parameters & no & unlabeled & no&yes \\
        SemiFed \cite{lin2021semifed} & model parameters & no & unlabeled & no &yes\\ 
        FedDF \cite{lin2020ensemble} & model parameters, soft labels & yes & unlabeled & no & yes \\ 
        DD \cite{bistritz2020distributed}& soft labels & yes & unlabeled & no &yes\\
        \pate \cite{papernot2016semi}& hard labels & yes & unlabeled & yes&no \\ 
        \textbf{FedCT (ours)} & \textbf{hard labels }& \textbf{yes} & \textbf{unlabeled} & \textbf{yes} &\textbf{yes} \\
        \hline
    \end{tabular}
    \end{adjustbox}
    \label{app:tab:federated_methods}
\end{table}

The main goal of \aimhi is to improve the privacy of current federated learning approaches while maintaining model quality. For that, we consider a classical FL scenario where clients hold a private local dataset. We additionally assume that they have access to a public unlabeled dataset.

In federated learning, clients share the parameters of local models (or model update / gradients) with a server that aggregates them. Averaging model parameters and gradients has been proposed first for maximum-entropy models~\citet{mcdonald2009efficient} and has been investigated in online learning~\citet{dekel2012optimal, kamp2014communication}. \citet{mcmahan2017communication} proposed averaging gradients and model parameters in deep learning (\fedavg). A variety of methods have improved the performance of \fedavg by altering the local objective function or the aggregation, in particular for heterogeneous data. For example, FedProx~\citep{li2020federated} adds a regularizer to the local loss function that promotes proximity to the last model aggregate, SCAFFOLD~\citep{karimireddy2020scaffold} adds control variables to the local optimization that account for client drift, MOON~\citep{li2021model} which uses a form of self-supervised learning to improve local models via a contrastive loss function, or FedGen~\citep{zhu2021data} which learns a data generator from the shared model parameters that generates additional training data for clients. While these methods improve performance, in particular on heterogeneous data, they require sharing model parameters (and sometimes additional information) and therefore have at least the same privacy drawbacks that \fedavg has.

Semi-supevised federated learning uses an unlabeled public dataset to share (additional) information. We discussed semi-supervised approaches that fit our scenario, i.e., that only share predictions on an unlabeled dataset, in Sec.~\ref{sec:related_work}.
There are, however, semi-supervised federated learning methods that do not directly fit our scenario or do not improve over the baselines we selected. 
For example, Fed-ET~\citep{cho2022heterogeneous},  semiFed~\citep{lin2021semifed}, SemiFL~\citep{diao2022semifl}, FedGen~\citep{zhu2021data}, and FedHKD~\citep{chen2023best} (also) share model parameters and therefore do not improve privacy over \fedavg. \citet{cho2023communication} propose to use co-regularization for personalized federated learning. For non-personalized FL this approach is equivalent to distributed distillation (\dd)~\citep{bistritz2020distributed}. The approach proposed by \citet{itahara2021distillation} is also similar to \dd.
FedMD~\citep{li2019fedmd} uses a public labeled dataset and therefore does not fit our scenario. Since the heterogeneous data setup they proposed is interesting, though, we nonetheless compare it to \fedct in App.~\ref{app:sec:compFedMD}. Similarly, \pate~\citet{papernot2016semi} is not collaborative and therefore also does not directly fit our scenario, as discussed in Sec.~\ref{sec:related_work}. Moreover, its privacy mechanism does not protect against an honest-but-curious server. However, since the approach is similar to \fedct in that it shares hard labels on an unlabeled dataset, we do compare \pate and its differentially private variant \dppate to \fedct in App.~\ref{app:PATE}.

In Tab.~\ref{app:tab:federated_methods} we provide an overview over the methods mentioned that highlights what information is shared by the method, whether it allows for model heterogeneity (i.e., whether clients may use different models), if and what type of public dataset is used, whether the method can train non-gradient-based models, such as decision trees, and whether the method is collaborative (i.e., whether clients learn from each other, or a single model is distilled from clients, instead).

\section{Additional Empirical Evaluation}
\label{app:extraexp}
In this section, we provide further details on our experiments, as well as additional results that further investigate the properties of \fedct. First, we demonstrate \fedct's ability to collaboratively train clients that each use different models, e.g., one using a decision tree, another a neural network, etc. We then provide additional details on our experiment on fine-tuning a large language model with \fedct. We also compare \fedct with two more distantly related baselines, \pate and FedMD: \pate also shares hard labels like \fedct, but is not collaborative, and FedMD shares soft labels and requires a labeled public dataset (see Tab.~\ref{app:tab:federated_methods} in Sec.~\ref{app:smifed:discussion}). To investigate the effect of co-training, we perform an ablation study in which we investigete the impact of the size of the unlabeled dataset on \fedct's performance and compare \fedct's performance to local training. Lastly, we provide an additional result on the scalability of \fedct in terms of the number of clients.

\subsection{Mixed Model Types}
\label{app:mix:model}
\begin{table}[ht!]
\centering
\caption{Mixed model experiment}
\begin{tabular}{c|c|c|c|c|c|c}
\hline
Dataset & C1 & C2 & C3 & C4 & C5 & ACC \\
\hline
BreastCancer & DT & RF & RuleFit & XGBoost & RF & $0.95 \pm 0.001$ \\
BreastCancer & DT & MLP & RuleFit & XGBoost & RF & $0.93 \pm 0.002$ \\
BreastCancer & XGBoost & XGBoost & XGBoost & XGBoost & XGBoost & $0.94 \pm 0.001$ \\
AdultsIncome &XGBoost&MLP&RF&RuleFit&XGBoost&$0.84 \pm 0.001$\\
\hline
\end{tabular}
\label{tab:mixmodel}
\end{table}

Sharing hard labels allows us to train any supervised learning method on each client. This means, we can even use different model types for different clients in \fedct. To demonstrate this, we compare using the best performing interpretable model on every client---for the BreatCancer dataset that is XGBoost---to two heterogeneous ensembles using different combinations of decision trees (DT), random forests (RF), rule ensembles (RuleFit), gradient-boosted decision trees (XGBoost), and neural networks (MLP). The results in Tab.~\ref{tab:mixmodel} show that using a diverse ensemble of models can further improve accuracy. That is, using only XGBoost models on each client results in a test accuracy of $0.94$ on the BreastCancer dataset, whereas using a heterogeneous ensemble yields a slightly higher test accuracy of $0.95$. Note that for this experiment we tried a few random combinations of models for the heterogeneous ensemble. Optimizing the ensemble to further improve accuracy makes for excellent future work.

\begin{figure}[t]
\begin{minipage}[t]{0.48\textwidth}
    \centering
    \includegraphics[width=0.88\linewidth]{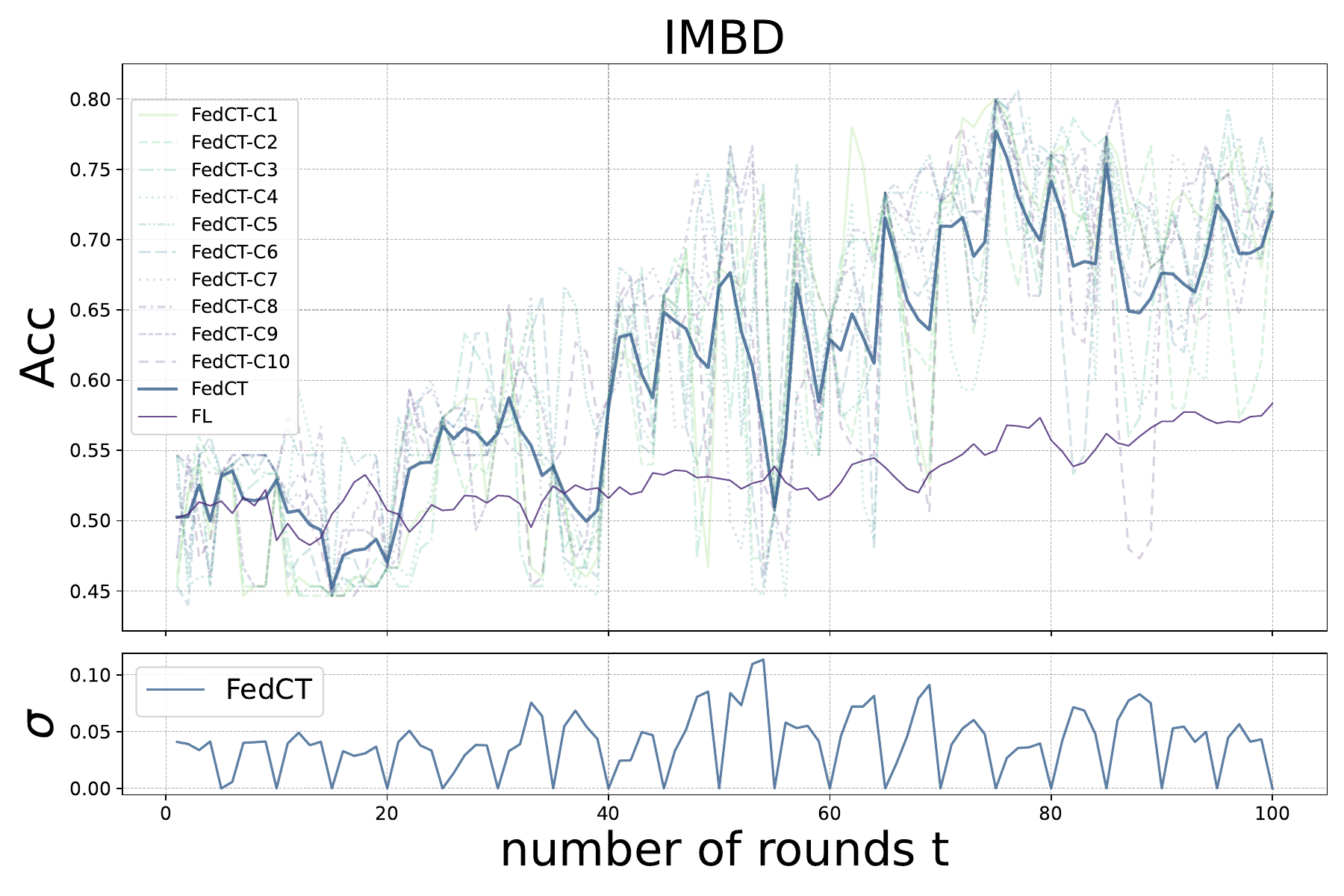}
    \caption{\textbf{Top}: Test accuracy (ACC) over time for GPT2 fine-tuning on IMBD with ACC of FL, and ACC of local models and their average for \fedct.\textbf{Bottom}: Standard deviation of test accuracy of local models in \fedct.}
    \label{fig:llm:IMBD}
\end{minipage}\hfill
\begin{minipage}[t]{0.48\textwidth}
    \centering
    \includegraphics[width=\linewidth]{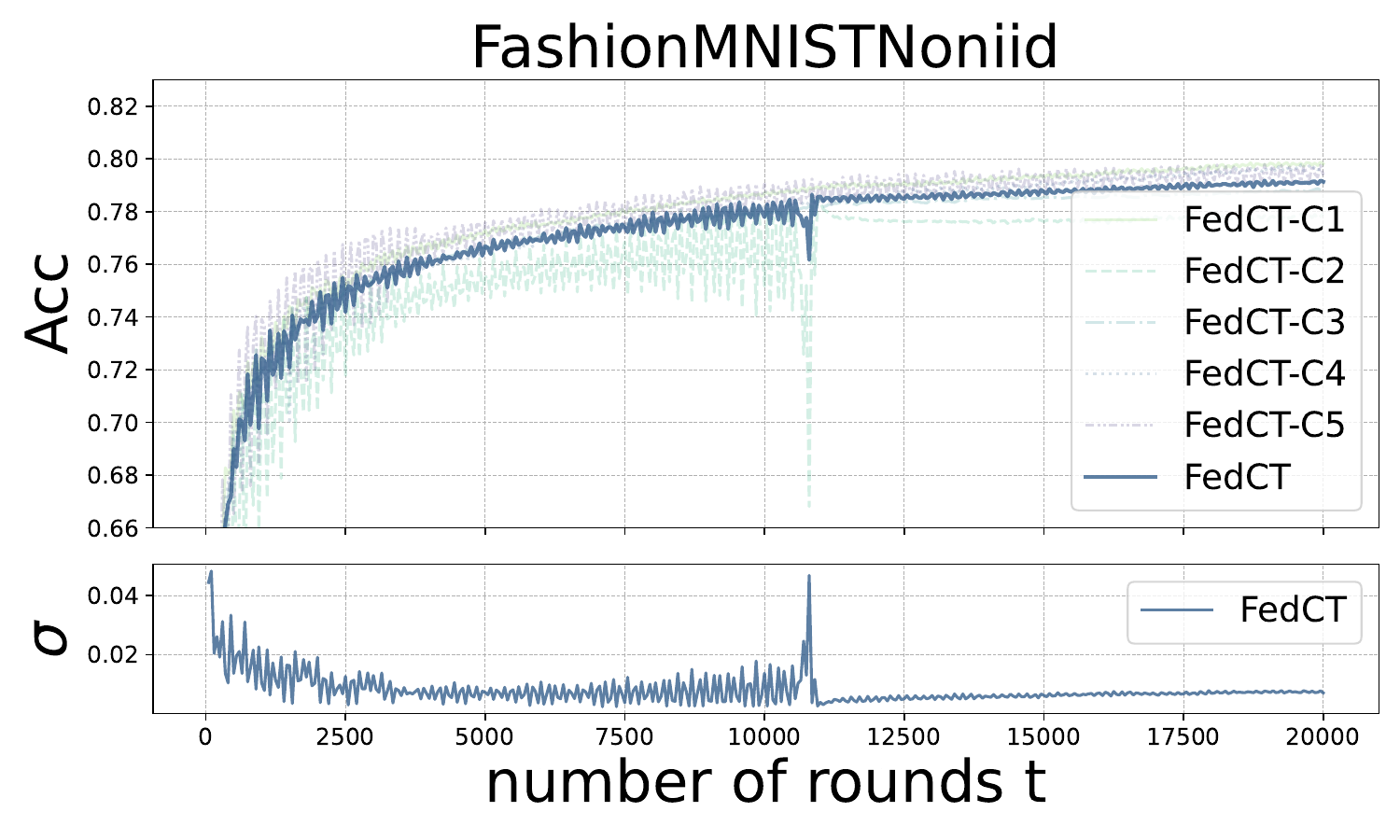}
    \caption{\textbf{Top}: Test accuracy (ACC) over time for $m=5$ local models of \fedct on heterogeneous distribution for the FashionMNIST dataset ($\alpha_{1}=100, \alpha_{2}=0.01$). \textbf{Bottom}: Standard deviation of test accuracy of local models in \fedct.}
    \label{app:fig:convnoniid}
\end{minipage}
\end{figure}

\subsection{Fine-Tuning Large Language Models (GPT2)}
\label{app:llm}

We now provide further details on our experiment on fine-tuning large language models via \fedct. The rational for this experiment was that fine-tuning should be a task particularly suitable to label sharing, since the quality of initial pseudo-labels will already be high, since model quality is good to begin with. In contrast, in standard learning scenarios the initial models will be of poor quality so that the first pseudo-labels will be fairly poor. 

To test whether \fedct performs well in a fine-tuning scenario, we fine-tune the GPT2 model transformer with a sequence classification head (linear layer) comprising of $124.44$ million parameters on the IMDB sentiment classification dataset~\citep{maas-EtAl:2011:ACL-HLT2011} and Twitter Sentiment Analysis dataset ~\citep{TwitterSentimentAnalysis}. For the IMBD dataset, we use $m=10$ clients each with a local dataset of $145$ examples and an unlabeled dataset of $|U|=150$ examples. We use a test set of $150$ examples for evaluation. We run fine-tuning for $T=100$ rounds with a communication period of $b=5$. Here, \fedct achieves a test accuracy of $0.73$, whereas \fedavg in that case only achieves a test accuracy of $0.59$. For Twitter, we used $m=10$ clients, a total of $39,682$ local training examples, an unlabeled dataset of $|U|=35,000$, and a testing set of $1000$ examples. \fedct achieved a testing accuracy of $0.65$, compared to $0.61$ for \fedavg.


We analyze the convergence in terms of test accuracy in Fig.~\ref{fig:llm:IMBD}. It shows that the initial accuracy is indeed already high, both for \fedct and \fedavg. Moreover, both methods exhibit a more linear convergence, in contrast to the saturating convergence behavior typical for standard learning tasks. The variance of individual clients in \fedct is low right from the start, but drops substantially in each communication round and increases again afterwards. This suggests that local datasets are heterogeneous enough to have different local optima, which is likely, given the small training set sizes. This data heterogeneity could explain the suboptimal performance of \fedavg, but leads to a conundrum: \fedct performed worse than \fedavg on strongly heterogeneous data. In the discussion (Sec.~\ref{sec:discussion}), we argued that an explanation why \fedct performs poorly on strongly heterogeneous data is that label quality on average is poor, since some clients have not seen a single example for some classes. If that explanation is correct, it would make sense that \fedct performs well in a fine-tuning setting, even under strong data heterogeneity: the quality of labels will be high.
As mentioned in the discussion, these results are promising and suggest that semi-supervised learning can be beneficial for fine-tuning in a federated setup. 

\subsection{Detailed Comparison to \pate}
\label{app:PATE}

\pate~\citep{papernot2016semi} is a distributed (but not collaborative) distillation algorithm with the goal of producing a single student model from an ensemble of teachers, each teacher trained on a distributed dataset. The approach assumes a secure environment for distributed dataset and server to produce pseudo-labels for a public unlabeled dataset. An untrusted entity can then distill a student model using the pseudo-labeling. This differs from the setting we consider, where the server that aggregates predictions and produces the pseudo-label cannot necessarily be trusted, i.e., we assume the server to be honest-but-curious. Although the setting of \pate fundamentally differs from the setting we consider, comparing \fedct and \pate is interesting: both approaches produce a pseudo-labeling via a majority vote and train models using these pseudo-labels on a public unlabeled dataset. A major difference is that \pate trains teachers until convergence first, then produces the pseudo-labels, and finally trains the student on the pseudo-labeled public dataset alone. In \fedct, clients iteratively train their models on the private training set in conjunction with the pseudo-labeled public dataset and update pseudo-labels in each communication round. Additionally, PATE provides privacy primarily on the server level rather than the client level, as the server aggregates non-private output from the teacher models and applies a privacy mechanism to the aggregate, but does not inherently protect the privacy of the data used by each teacher during the training. Therefore, we expect \fedct to achieve higher model quality than \pate. 

It is noteworthy that \pate also offers a differential privacy mechanism on the shared information---in contrast to \fedct, this shared information is not the local predictions, but rather the consensus labels produced by the server. \pate achieves differential privacy by using a variant of the Gaussian mechanism~\citep{dwork2006calibrating} on the prediction counts for each class, i.e., for each example the prediction counts are the numbers of votes for a particular class from the clients. These counts are integer values, but not binary, and therefore a correctly calibrated Gaussian noise achieves R\'enyi differential privacy (see Thm. 6 on the GNMax Aggregator in~\citep{papernot2016semi}). We compare both \fedct and \dpfedct to \pate and its differentially private variant \dppate. The main comparison is presented in Tab.~\ref{table:iidexp} in Sec.~\ref{sec:experiments}, where they are compared to all other baselines. We summarize the results only for \fedct and \pate in Tab.~\ref{tab:PATE1} for $m=5$  clients/teachers. The results show that, indeed, collaborative training (\fedct) achieves substantially higher test accuracy than distillation (\pate). Surprisingly, sharing hard labels of each client (\fedct) results in a slightly lower vulnerability score (VUL) than sharing the pseudo-labels (\pate). One possible explanation is that in \pate, teachers are trained until full convergence which might lead teachers to overfit the local training data. Therefore, their shared information might reveal more about the training data, since membership attacks tend to be more effective when a model is overfitted \cite{shokri2017membership}. 
We furthermore observe that both \fedct and \pate are robust to their respective differential privacy mechanisms, since both \dpfedct and \dppate achieve a test accuracy similar to their non-differentially-private variants. 

In Tab.~\ref{tab:PATE2} we investigate the scalability of \fedct and \pate by testing both methods on $m=100$ clients on the FashionMNIST and Pneumonia datasets. The results indicate that both methods achieve high accuracy also with a higher distribution of datasets and that the advantage of collaborative training (\fedct) over simple distillation (\pate) remains.

\begin{table}[t]
\centering
\caption{Test acc. (ACC) and privacy vulnerability (VUL, smaller is better) for $5$ clients on iid data.}
\begin{adjustbox}{width=1\textwidth}
\small
\begin{tabular}{c|cc|cc|cc|cc|cc}
\hline 
\textbf{Method} & \multicolumn{2}{c|}{\textbf{CIFAR10}} & \multicolumn{2}{c|}{\textbf{FashionMNIST}} & \multicolumn{2}{c|}{\textbf{Pneumonia}} & \multicolumn{2}{c|}{\textbf{MRI}} & \multicolumn{2}{c}{\textbf{SVHN}} \\
 & ACC & VUL & ACC & VUL & ACC & VUL & ACC & VUL & ACC & VUL \\
\hline
\textbf{\fedct} & $\mathbf{0.77}\pm 0.003$ & $0.52$ & $\mathbf{0.84}\pm 0.004$ & $\mathbf{0.51}$ & $\mathbf{0.78}\pm 0.008$ & $\mathbf{0.51}$ & $0.64\pm 0.004$ & $0.52$ & $\mathbf{0.91}\pm 0.002$ & $\mathbf{0.53}$ \\
\textbf{\dpaimhi} $(\epsilon=0.1)$ & $0.76 \pm 0.002$ & $\mathbf{0.51}$ & $0.80 \pm 0.001$ & $0.52$ & $0.75 \pm 0.004$ & $\mathbf{0.51}$ & $0.62 \pm 0.002$ & $\mathbf{0.51}$ & $0.86\pm 0.001$ & $\mathbf{0.53}$ \\
\textbf{\pate} & $0.69\pm 0.002$ & $0.60$ & $0.73\pm 0.001$ & $0.59$ & $0.75\pm 0.003$ & $0.59$ & $0.61\pm 0.001$ & $0.60$ & $0.87\pm 0.002$ & $0.58$ \\
\textbf{\dppate} & $0.67 \pm 0.003$ & $0.58$ & $0.73 \pm 0.002$ & $0.57$ & $0.71\pm 0.001$ & $0.58$ & $0.60 \pm 0.001$ & $0.57$ & $0.86 \pm 0.002$ & $0.57$ \\
\hline
\end{tabular}
\end{adjustbox}
\vspace{0.1cm}
\label{tab:PATE1}
\end{table}

\begin{table}[ht]
\centering
\caption{Test accuracy ACC of \pate Vs \aimhi for $m=100$ clients}
\begin{tabular}{c|c|c}
\hline
\textbf{Dataset}&\textbf{\fedct}&\textbf{\pate} \\
\hline
FashionMNIST & 0.7154 & 0.6318 \\
Pneumonia & 0.7269 & 0.6903 \\
\hline
\end{tabular}
\label{tab:PATE2}
\end{table}

\subsection{Comparison to \fedmd}
\label{app:sec:compFedMD}
We evaluate the performance of \fedct on a scenario with heterogeneous data proposed by \citet{li2019fedmd}. This scenario is not compatible with our assumptions, since it assumes a public labeled dataset and it assumes that for both the private and public dataset each instance has two labels: a fine-grained class label, and a more coarse superclass label. An example of such a dataset is CIFAR100, where the $100$ classes fall into $20$ superclasses. Data heterogeneity is achieved by homogeneously distributing superclasses over clients, but in such a way that each client only observes a single class per superclass. This means, while all clients observe vehicles, some only observe cars, others only bicycles. The classification task is to predict the superclass. This should ensure that a meaningful consensus can be achieved. So while \fedct is not designed for this scenario, we can apply it here by ignoring the labels on the public dataset and only using the superclass labels on private data. We compare \fedct to the method \citet{li2019fedmd} proposed for this scenario (\fedmd). Here, \fedct achieves a test accuracy of $0.51$, slightly outperforming \fedmd which achieves a test accuracy of $0.50$. We do not compare privacy, since \fedmd is algorithmically equivalent to \dd using a central server.

\subsection{Scalability:}
\label{app:scalability}

We compare the scalability  of \fedct in terms of the number of clients to \fedavg on FashionMNIST, using the same setup as before.  We increase the number of clients $m\in\{5,10,20,40,80\}$ and keep the overall training set size constant, so for larger numbers of clients the local training set size decreases. The results in Fig.~\ref{fig:scalability} show that higher levels of distribution reduce the accuracy slightly, but 
both \fedct and \fedavg show only a moderate decline, with \fedavg performing slightly better than \fedct. To ensure the positive scalability results are not an outlier particular to the FashionMNIST dataset, we also reported the test accuracy wrt. the number of clients on the Pneumonia dataset, as well. The results in Fig.~\ref{fig:scalability_pnumonia} show that also on this dataset \fedct scales as well as \fedavg with the number of clients. This indicates that \fedct indeed scales well with the number of clients.

\begin{figure}[t]
\begin{minipage}[t]{0.48\textwidth}
    \centering
    \includegraphics[width=\linewidth]{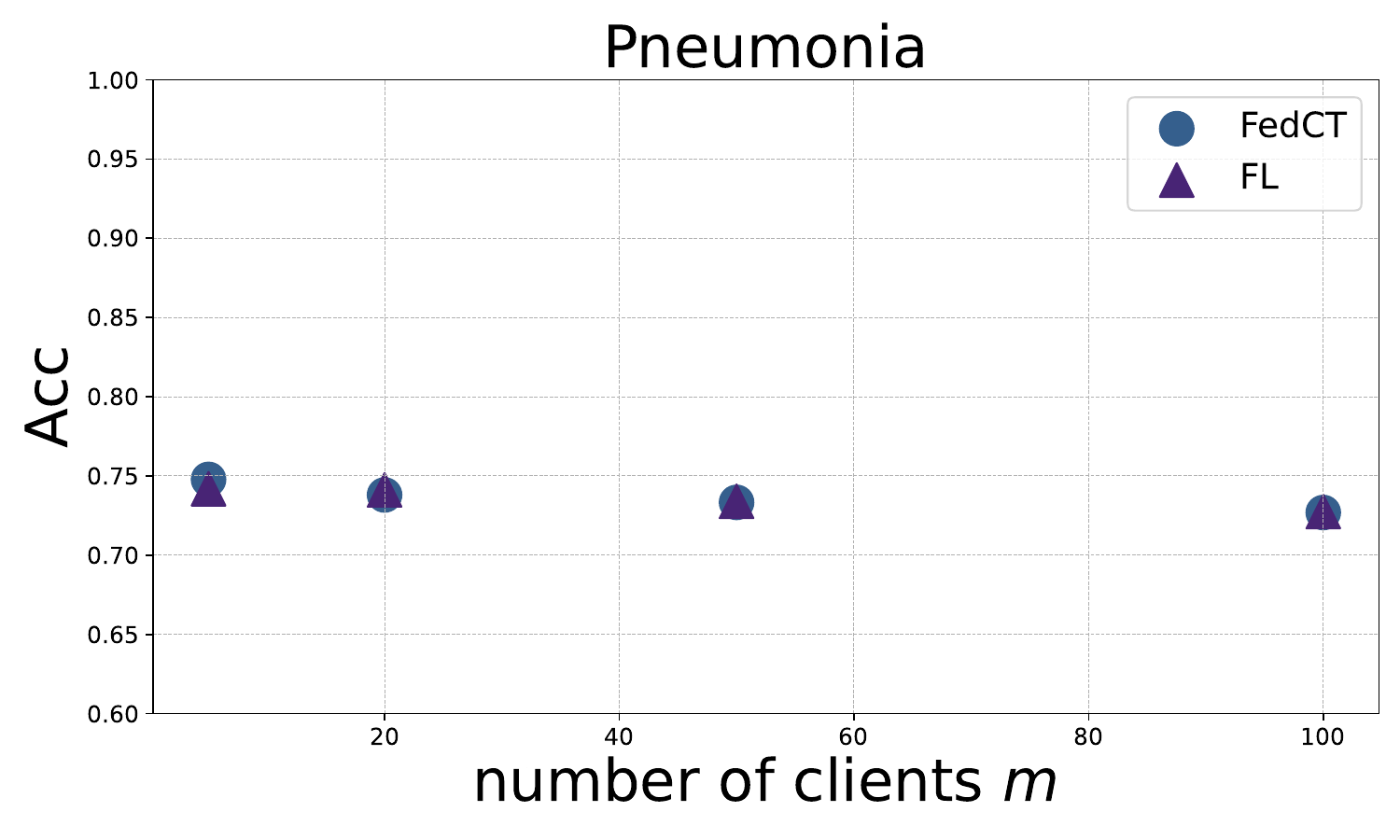}
    \caption{Test accuracy (ACC) of \fedct and \fedavg (FL) on Pneumonia with $|U|=200$ for various numbers of clients $m$.}
    \label{fig:scalability_pnumonia}
\end{minipage}\hfill
\begin{minipage}[t]{0.48\textwidth}
    \centering
    \includegraphics[width=\linewidth]{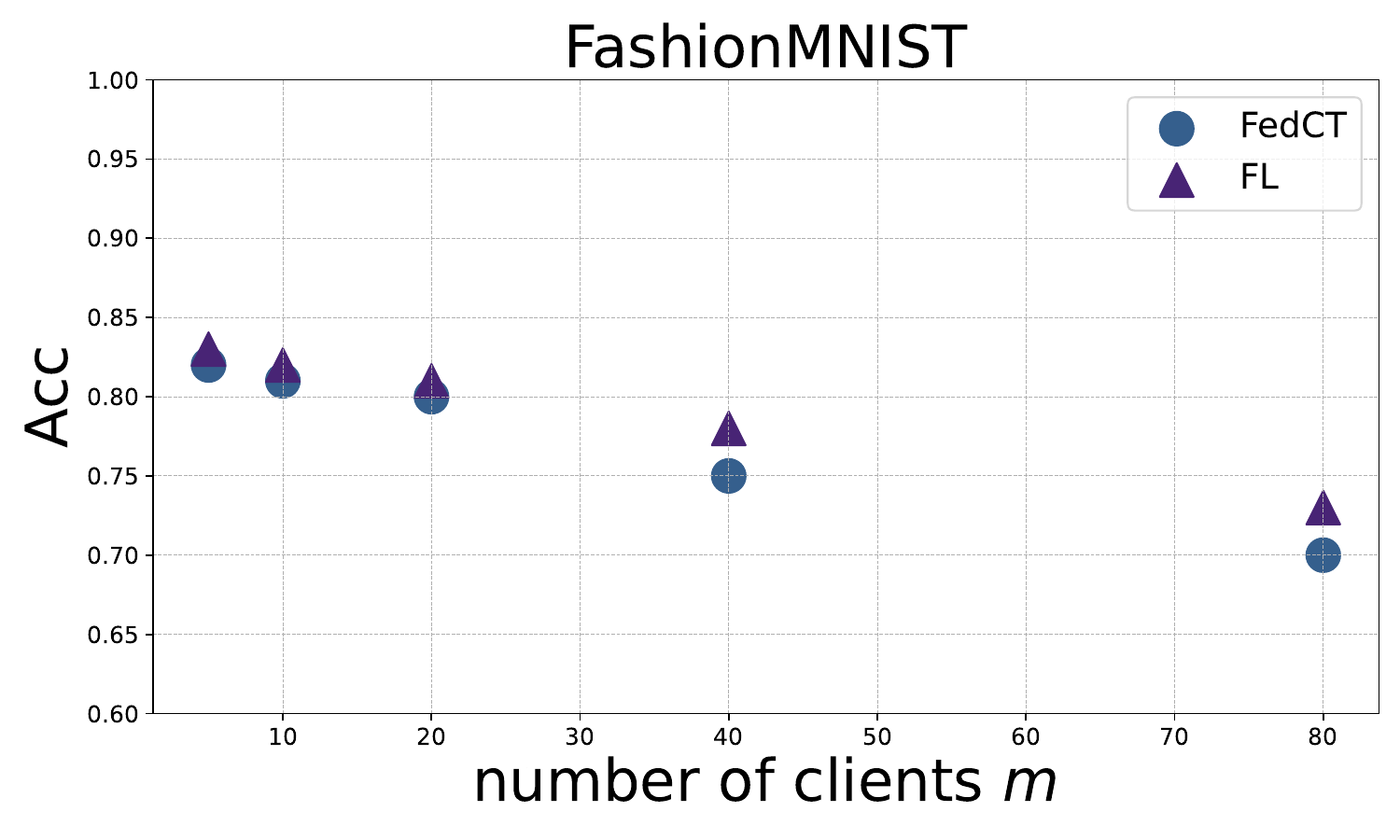}
    \caption{Test acc. (ACC) of \fedct ($|U|=5\cdot10^5$) and \fedavg on FashionMNIST vs. numbers of clients $m$.\vspace{-0.3cm}}
    \label{fig:scalability}
\end{minipage}
\end{figure}

\subsection{Detailed Analysis of Heterogeneous Data Distributions}
\label{app:deepnoniid}
We observed that using a 90$\%$ qualified majority vote in \fedct improves performance on pathological data distributions compared to a simple majority vote. To explain this improvement, we tracked the consensus label quality in each communication round. For that we use the fact that for our benchmark datasets, we have a ground truth label also for the unlabeled data, which is not available to \fedct. With this, we can assess the quality of consensus labels by measuring their accuracy against the ground truth on the unlabeled dataset. 

Note that using a qualified majority has two effects: It ensures that a consensus label is supported by a large quorum of clients, but it also reduces the size of the pseudo-labeled dataset, since examples for which there is no qualified majority are discarded in that communication round. 
Instead, in a majority vote the full unlabeled dataset is added to each client's local training data. 

In pathological scenarios ($\alpha_{1}= \alpha_{2} =0.01$), the majority tends to be incorrect due to the high heterogeneity in data distribution. Intuitively, that is reasonable, since clients only observe a small subset of labels, so for each example in $U$ the majority of clients has not observed its true label in their local training sets. In Fig.\ref{fig:app:noniid_QMaccuracyOfPseudolabels} we see that the quality of the consenssus labels for majority voting with $\alpha_1=\alpha_2=0.01$ is poor, which explains the suboptimal performance of \fedct in this case. Using a $90\%$ qualified majority, however, substantially improves the quality of pseudo-labels, resulting in the better performance observed for $\fedct(QM)$. In Fig.~\ref{fig:app:noniid_QMnumberOfExamplesTaken} we see that the higher quality in consensus labels comes at the price of fewer unlabeled examples being available to clients, since most examples are discarded. Clients do improve through these few, high-quality examples, though: after a few communication rounds a qualified quorum is found for almost all unlabeled examples.

For medium heterogeneity ($\alpha_{1}=100, \alpha_{2}=0.01$), using a qualified majority does not improve performance over majority voting. Fig.~\ref{fig:app:noniid_QMaccuracyOfPseudolabels} shows that in this case the quality of the qualified majority is similar to that of normal majority voting, yet the number of examples with a consensus label is substantially smaller when using a qualified majority, as seen in Fig.~\ref{fig:app:noniid_QMnumberOfExamplesTaken}. In this case, the disadvantage of QM taking fewer examples is not outweighed by the quality improvement on the labels. 

We conclude that using a qualified majority vote as consensus mechanism successfully mitigates the challenges of \fedct on pathological non-iid datasets due to the improved quality of pseudo-labels. It would make for excellent future work to investigate how the consensus mechanism can be further improved, e.g., by incorporating the confidence of local models on their predictions.

\begin{figure}[t]
\begin{minipage}[t]{0.48\textwidth}
    \centering
    \includegraphics[width=\linewidth]{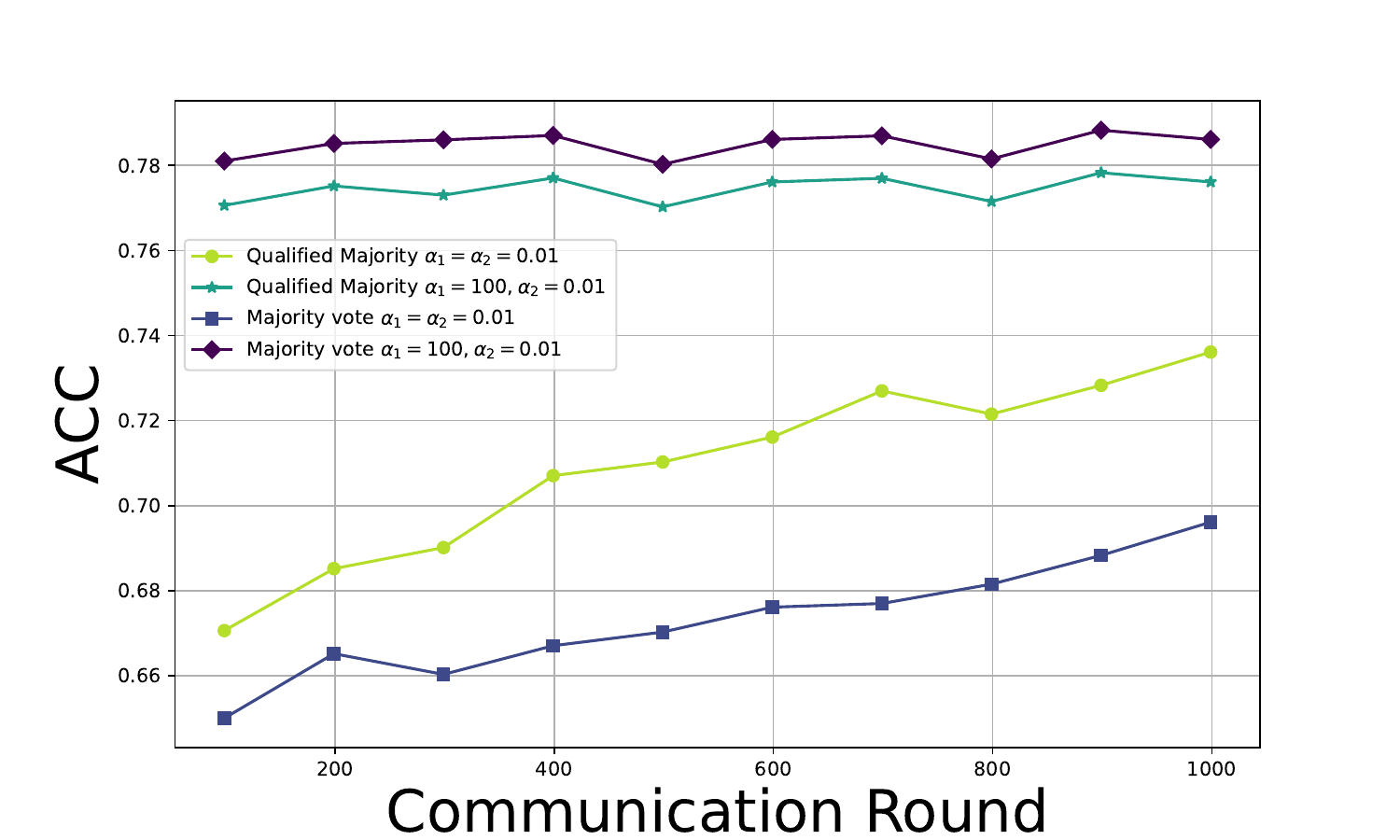}
    \caption{The ground-truth accuracy ACC of the consensus labels in every communication round.}
    \label{fig:app:noniid_QMnumberOfExamplesTaken}
\end{minipage}\hfill
\begin{minipage}[t]{0.48\textwidth}
    \centering
    \includegraphics[width=\linewidth]{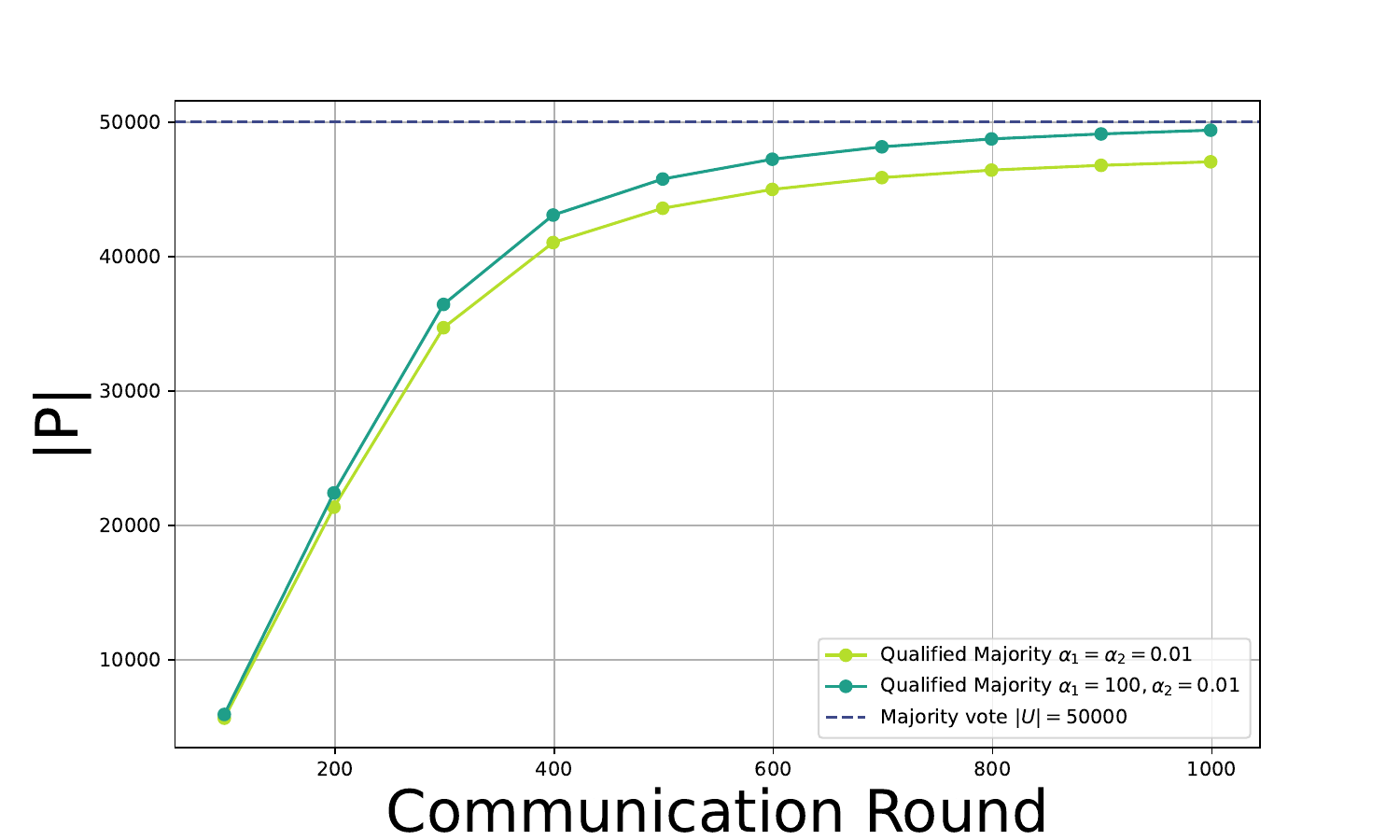}
    \caption{The size of pseudo-labeled public dataset used in every communication round.}
    \label{fig:app:noniid_QMaccuracyOfPseudolabels}
\end{minipage}
\end{figure}

\begin{table}
\centering
\caption{CIFAR100 classes that are semantically more similar to the CIFAR10 classes.}
\small
\begin{tabular}{c|c}
\hline 
\textbf{CIFAR10 class} & \textbf{CIFAR100 classes}\\
\hline
airplane & rocket \\
automobile & bus, streetcar, tank, lawn\_mower \\
bird & bee, butterfly \\
cat & tiger, lion, leopard \\
deer & cattle \\
dog & wolf, fox \\
frog & lizard \\
horse & camel \\
ship & sea \\
truck & pickup\_truck \\
\hline
\end{tabular}
\label{table:classmappingC100C10}
\end{table}

\subsection{Abblation Study}
\label{app:abblation:study}
We want to investigate the effect of co-training on the performance. For that, we first analyze the effect of sharing predictions on the unlabeled dataset by evaluating the performance of \fedct on varying sizes of $U$. We then investigate the effect of collaboration by comparing the performance of \fedct and \fedavg to local training without communication.

\paragraph{Effect of the Size of the Unlabeled Dataset:}
\label{app:effect:unlabeled}

Since \fedct utilizes a public unlabeled dataset, a natural question is how large this unlabeled dataset needs to be. To answer this question, we evaluate the performance of FedCT for varying sizes $|U|$ of the unlabeled dataset. We evaluate \fedct on the Pneumonia dataset where we fixed the local training data set size to 100 examples. Our results in Fig.~\ref{app:fig:unlabeled:size} show that with $|U|=100$ examples we reach high model quality, comparable to \fedavg, and that further increasing $|U|$ only slightly improves model quality. Using too little unlabeled examples, however, does negatively impact \fedct. This indicates, that \fedct indeed utilizes the unlabeled dataset. To further investigate this, we compare \fedct also to training local models without communication and pooling all data for centralized training. Indeed, \fedct even with a small unlabeled dataset outperforms local training without communication. Moreover, with $|U|=1000$ \fedct nearly reaches the model quality of centralized training.

\begin{figure}[H]
    \centering
    \includegraphics[width=0.55\linewidth]{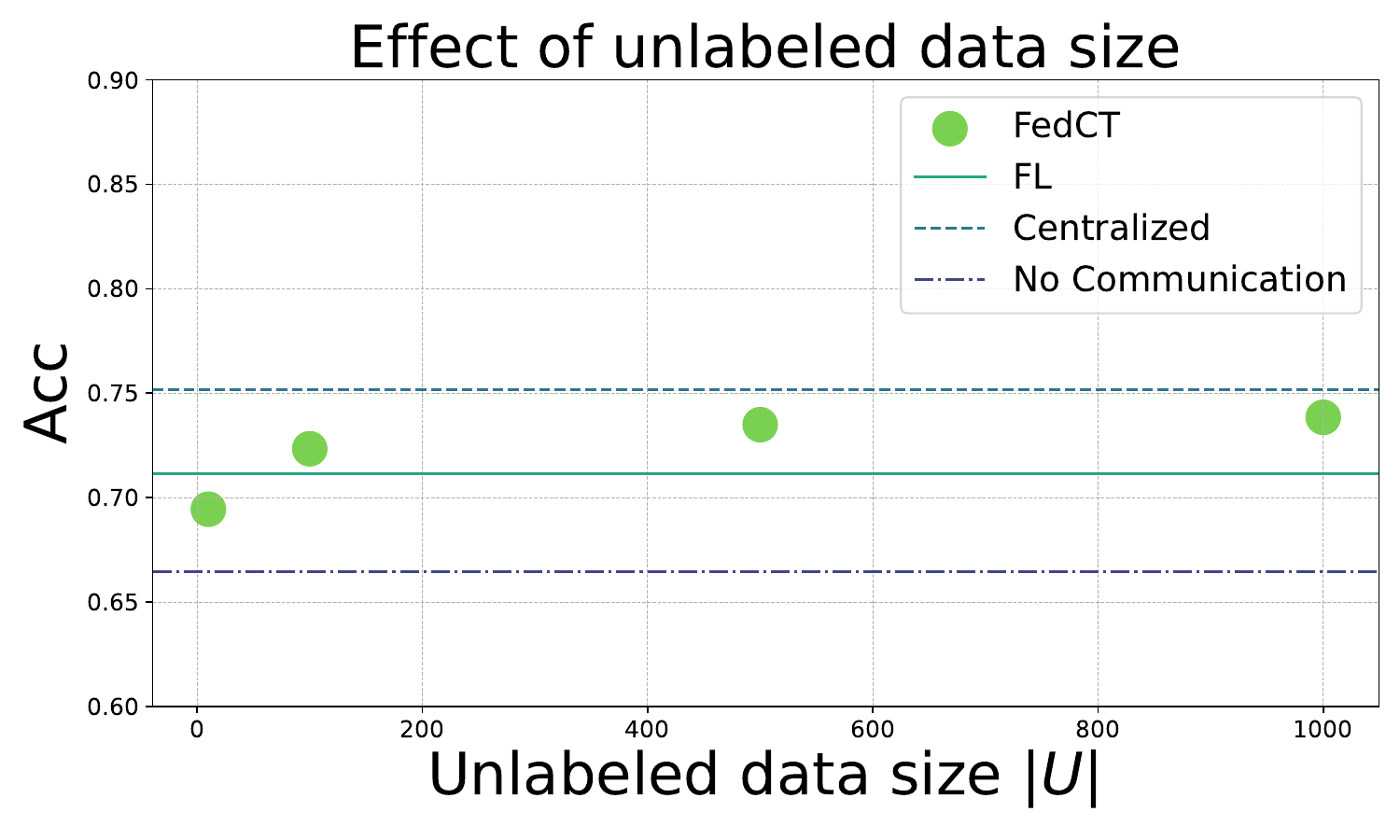}
    \caption{Test accuracy (ACC) of \fedct under different unlabeled dataset size $U$.}
    \label{app:fig:unlabeled:size}
    \vspace{-0.3cm}
\end{figure}

\paragraph{Effect of Collaboration:}
\label{app:noncomstudy}
Our experiments so far show that the collaboration via co-training is as effective as the collaboration via model averaging. A natural question then is: is collabortion necessary at all? In particular, we have argued that the good performance of \dpfedct with a lot of noise on the communicated local predictions is due to the robustness of majority voting. Another explanation would be, however, that the performance of both \fedct and \fedavg do not depend on collaboration at all. Instead, local training is sufficient to achieve the reported performances. To ensure that collaboration is beneficial, we compare \fedct and \fedavg to only local training without any communication. We use the same setup as in our main experiment from Tab.~\ref{table:iidexp} and report the average test accuracy of local models over three runs in Tab.~\ref{table:nocom}. These results show that collaboration substantially improves test accuracy over no-communication---on average by $0.11$. Therefore, collaboration is indeed beneficial. This supports our hypothesis that the good performance of \dpfedct even under strong DP noise is due to the robustness of majority voting to noise, as \citet{papernot2016semi} suggested as well.


\begin{table*}[ht]
\centering
\caption{Test accuracy (ACC)  for $m=5$ clients on iid data.}
\begin{adjustbox}{width=0.6\textwidth}
\small
\begin{tabular}{c|c|c|c}
\hline 
\textbf{Dataset} &\textbf{no-comunication} &\textbf{\fedct}& \textbf{\fedavg} \\
\hline
CIFAR10 & $0.73 \pm 0.003$&$0.77\pm0.003$&$0.77\pm0.020$\\
FashionMNIST &$0.79\pm 0.005$ &$0.84\pm 0.004$&$0.83\pm0.024$\\
Pneumonia & $0.68 \pm 0.004$ &$0.78 \pm 0.008$&$0.74 \pm 0.013$\\
MRI & $0.42 \pm 0.001$ &$0.63 \pm 0.004$&$0.66 \pm 0.015$\\
SVHN &$0.76 \pm 0.002$ &$0.91 \pm 0.002$&$0.91 \pm 0.026$\\
\hline
\end{tabular}
\end{adjustbox}
\label{table:nocom}
\end{table*}

\paragraph{\fedavg with access to public Unlabeled dataset:} 
In \fedavg,\dpfl, and \fedmd, the clients use the same private dataset as \fedct,\pate, and \dd. However, \fedct,\pate and \dd also leverage access to a public unlabeled dataset. This setup reflects a realistic scenario in many real-world applications. For instance, in healthcare, public health databases often contain relevant patient information but lack labels for specific diseases ( e.g., The TCGA database provides labels for only 33 out of more than 200 known cancer types). A natural question arises: How much could \fedavg improve if clients had the resources to manually label portions of the public unlabeled dataset? To explore this, we simulated such a scenario by distributing the unlabeled dataset in CIFAR10 ($|U|=10k$) with correct labels to local clients in \fedavg. With this additional labeled data, \fedavg's performance improved slightly, from $0.77\pm 0.02$ to $0.78\pm 0.03$. However, as with many semi-supervised approaches, the model performance gain often comes at the prohibitive cost of manually labeling $10k$ examples. This experiment demonstrates that even if clients could label the public dataset, our approach achieves comparable performance while offering substantially better privacy.

\subsection{A Note on the Byzantine Resilience of \fedct}
\label{app:mix:Byzantine-resilience}

\citet{jiang2020federated} argue that federated semi-supervised learning with soft label sharing is more Byzantine robust than \fedavg. On the example of \feddistill they argue that the thread vector is bounded on the probability simplex over classes---unlike \fedavg, where the thread vector is unbounded---which makes soft-label sharing more resilient against Byzantine attacks. Following this logic, we argue that the thread vector in \fedct is even more bounded (to a binary vector over classes) and therefore we expect \fedct to be at least as Byzantine resilient as \feddistill. Investigating the Byzantine resilience of hard label sharing further, in particular in combination with qualified majority voting, is an interesting direction for future work.

\section{Details on Experiments}
\label{app:exp:details}

\subsection{Details on Privacy Vulnerability Experiments}
\label{app:privacy:details}
We measure privacy vulnerability by performing membership inference attacks against \aimhi, \fedavg, \dd, and \pate. 
In all attacks, the attacker creates an attack model using a model it constructs from its training and test datasets. Similar to previous work~\cite{shokri2017membership}, we assume that the training data of the attacker has a similar distribution to the training data of the client. Once the attacker has its attack model, it uses this model for membership inference using ShadowMetric developed by \cite{shokri2017membership}. 
In blackbox attacks (in which the attacker does not have access to intermediate model parameters), it only uses the classification scores it receives from the target model (i.e., client's model) for membership inference. 
On the other hand, in whitebox attacks (in which the attacker can observe the intermediate model parameters), it can use additional information in its attack model. 
Since the proposed \aimhi,\dd and \pate do not reveal intermediate model parameters to any party, it is only subject to blackbox attacks. Vanilla federated learning on the other hand is subject to whitebox attacks. 
Each inference attack produces a membership score of a queried data point, indicating the likelihood of the data point being a member of the training set. We measure the success of membership inference as ROC AUC of these scores. 
The \textbf{vulnerability (VUL)} of a method is the ROC AUC of membership attacks over $K$ runs over the entire training set (also called attack epochs) according to the attack model and scenario. A vulnerability of $1.0$ means that membership can be inferred with certainty, whereas $0.5$ means that deciding on membership is a random guess.

We assume the following attack model: clients are honest and the server may be semi-honest (follow the protocol execution correctly, but it may try to infer sensitive information about the clients). The main goal of a semi-honest server is to infer sensitive information about the local training data of the clients. This is a stronger attacker assumption compared to a semi-honest client since the server receives the most amount of information from the clients during the protocol, and a potential semi-honest client can only obtain indirect information about the other clients. We also assume that parties do not collude.

The attack scenario for \aimhi, \dd and \pate is that the attacker can send a (forged) unlabeled dataset to the clients and observe their predictions, equivalent to one attack epoch ($K=1$); the one for \fedavg and \dpfl is that the attacker receives model parameters and can run an arbitrary number of attacks---we use $K=500$ attack epochs.

\subsection{Datasets}
\begin{table}[b]
\centering
\caption{Dataset descriptions for image classification experiments.}
\begin{adjustbox}{width=1\textwidth}
\small
\begin{tabular}{|c|c|c|c|c|c|c}
\hline 
\multicolumn{1}{|c|}{Dataset}& training size& testing size & unlabeled size $|U|$& communication period $b$ & number of rounds $T$  \\
\hline
 CIFAR10 & $40\cdot 10^3$&$10\cdot 10^3$ &$10\cdot 10^3$ &$10$ & $3\cdot 10^3$ \\
 FashionMNIST& $10\cdot 10^3$&$10\cdot 10^3$ & $50\cdot 10^3$& $50$& $20\cdot 10^3$ \\
 Pneumonia & $4386$& $624$&$900$ & $20$& $20\cdot 10^3$\\
 MRI&$30$&$53$ &$170$ & $6$&$2\cdot 10^3$ \\
 SVHN & $38\,257$&$26\,032$ & $35\cdot 10^3$& $10$& $20\cdot 10^3$  \\
 IMDB & $1450$&$150$ & $150$& $5$& $100$  \\
\hline
\end{tabular}
\end{adjustbox}
\vspace{0.1cm}
\label{tbl:expsetup}
\end{table}
%
We use $3$ standard image classification datasets: CIFAR10~\citep{krizhevsky2010cifar}, FashionMNIST~\citep{xiao2017/online}, IMDB \cite{maas-EtAl:2011:ACL-HLT2011} and SVHN~\citep{netzer2011reading}. We describe the datasets and our preprocessing briefly.

\textit{CIFAR10} consists of $50\,000$ training and $10\,000$ test $32\times 32$ color images in $10$ classes with equal distribution (i.e., a total of $6\,000$ images per class). Images are normalized to zero mean and unit variance.
\textit{FashionMNIST} consists of $60\,000$ training and $10\,000$ test $28\times 28$ grayscale images of clothing items in $10$ classes with equal distribution. Images are not normalized.
\textit{SVHN} (Street View House Numbers) consists of $630\,420$ $32\times 32$ color images of digits from house numbers in Google Street View, i.e., $10$ classes. The datasest is partitioned into $73\,257$ for training, $26\,032$ for testing, and $531\,131$ additional training images. In our experiments, we use only the training and testing set. Images are not normalized.

\textit{IMBD} sentimental dataset consists of a large collection of movie reviews along with corresponding sentiment labels indicating whether the sentiment expressed in each review is positive or negative. A total of $1750$ examples have been used. $1450$ examples as a training set, $150$ examples as a testing set and $150$ examples as unlabeled dataset.

We use five standard datasets from the UCI Machine Learning repository for our experiments on collaboratively training interpretable models: WineQuality~\citep{cortez2009modeling}, BreastCancer~\citep{sudlow2015uk}, AdultsIncome~\citep{misc_adult_2}, Mushroom~\citep{misc_mushroom_73}, and Covertype~\citep{misc_covertype_31}. A short description of the five datasets follows. 
\textit{WineQuality} is a tabular dataset of $6\,497$ instances of wine with $11$ features describing the wine (e.g., alcohol content, acidity, pH, and sulfur dioxide levels) and the label is a wine quality score from $0$ to $10$. We remove duplicate rows and transform the categorial type attribute to a numerical value. We then normalize all features to zero mean and unit variance.
\textit{BreastCancer} is a medical diagnostics tabular dataset with $569$ instances of breast cell samples with $30$ features describing cell nuclei with $2$ classes (malignant and benign). We followed the same preprocessing steps as WineQuality dataset.
\textit{AdultIncome} is a tabular dataset with  $48,842$ instances of adults from various backgrounds with $14$ features describing attributes such as age, work class, education, marital status, occupation, relationship, race, gender, etc. The dataset is used to predict whether an individual earns more than $50,000\$$ a year, leading to two classes: income more than $50,000\$$, and income less than or equal to $50,000\$$.
\textit{Mushroom} is a biological tabular dataset with  $8124$ instances of mushroom samples with $22$ features describing physical characteristics such as cap shape, cap surface, cap color, bruises, odor, gill attachment, etc. The dataset is used to classify mushrooms as edible or poisonous, leading to two classes: edible and poisonous.
\textit{Covertype} is an environmental tabular dataset with $581,012$ instances of forested areas with $54$ features describing geographical and cartographical variables, such as elevation, aspect, slope, horizontal distance to hydrology, vertical distance to hydrology, horizontal distance to roadways, hillshade indices, and wilderness areas and soil type binary indicators. The dataset is used to predict forest cover type, leading to  $7$ distinct classes: Spruce/Fir, Lodgepole Pine, Ponderosa Pine, Cottonwood/Willow, Aspen, Douglas-fir, and Krummholz.

Furthermore, we use $2$ medical image classification datasets, Pneumonia~\citep{kermany2018identifying}, and MRI\footnote{\url{https://www.kaggle.com/datasets/navoneel/brain-mri-images-for-brain-tumor-detection}}. 
\textit{Pneumonia} consists of $5\,286$ training and $624$ test chest x-rays with labels \textit{normal}, \textit{viral pneumonia}, and \textit{bacterial pneumonia}. We simplify the labels to \textit{healthy} and \textit{pneumonia} with a class imbalance of roughly $3$ pneumonia to $1$ healthy. The original images in the Pneumonia dataset do not have a fixed resolution as they are sourced from various clinical settings and different acquisition devices. We resize all images to a resolution of $224\times 224$ pixels without normalization.
\textit{MRI} consists of $253$ MRI brain scans with a class imbalance of approximately $1.5$ brain tumor scans to $1$ healthy scan. Out of the total $253$ images, we use $53$ images as testing set. Similar to the pneumonia dataset, the original images have no fixed resolution and are thus resized to $150\times 150$ without normalization.

\subsection{Experimental Setup}
\label{app:exp:setup}
%
\begin{table}[t]
\centering
\caption{CIFAR10 architecture}
\begin{tabular}{|c|c|c|c|}
\hline
\textbf{Layer} & \textbf{Output Shape} & \textbf{Activation} & \textbf{Parameters} \\
\hline
Conv2D & (32, 32, 32) & ReLU & 896 \\
BatchNormalization & (32, 32, 32) & - & 128 \\
Conv2D & (32, 32, 32) & ReLU & 9248 \\
BatchNormalization & (32, 32, 32) & - & 128 \\
MaxPooling2D & (16, 16, 32) & - & - \\
Dropout & (16, 16, 32) & - & - \\
Conv2D & (16, 16, 64) & ReLU & 18496 \\
BatchNormalization & (16, 16, 64) & - & 256 \\
Conv2D & (16, 16, 64) & ReLU & 36928 \\
BatchNormalization & (16, 16, 64) & - & 256 \\
MaxPooling2D & (8, 8, 64) & - & - \\
Dropout & (8, 8, 64) & - & - \\
Conv2D & (8, 8, 128) & ReLU & 73856 \\
BatchNormalization & (8, 8, 128) & - & 512 \\
Conv2D & (8, 8, 128) & ReLU & 147584 \\
BatchNormalization & (8, 8, 128) & - & 512 \\
MaxPooling2D & (4, 4, 128) & - & - \\
Dropout & (4, 4, 128) & - & - \\
Flatten & (2048,) & - & - \\
Dense & (128,) & ReLU & 262272 \\
BatchNormalization & (128,) & - & 512 \\
Dropout & (128,) & - & - \\
Dense & (10,) & Linear & 1290 \\
\hline
\end{tabular}
\label{tbl:CIFAR10arch}
\end{table}
We now describe the details of the experimental setup used in our empirical evaluation. All experiments were conducted on a cluster node with $6$ NVIDIA RTX A6000 with $48$GB VRAM, an AMD EPYC 7402 24-Core Processor and $1$TB RAM.

In our privacy-utility trade-off experiments, we use $m=5$ clients for all datasets. We report the split into training, test, and unlabeled dataset per dataset, as well as the used communication period $b$ and number of rounds $T$ in Tab.~\ref{tbl:expsetup}.
For the scalability experiments, we use the same setup, varying $m\in\{5,10,20,40,80\}$ clients.
For the experiments on heterogeneous data distributions, we use the same setup as for the privacy-utility trade-off, but we sample the local dataset from a Dirichlet distribution as described in the main text.


For all experiments, we use Adam as an optimization algorithm with a learning rate $0.01$ for CIFAR10, and $0.001$ for the remaining datasets. A description of the DNN architecture for each dataset follows.

The neural network architectures used for each dataset are given in the following. 
For CIFAR10 we use a CNN with multiple convolutional layers with batch normalization and max pooling. The details of the architecture are described in Tab.~\ref{tbl:CIFAR10arch}.
For FashionMNIST, we use a simple feed forward architecture on the flattened input. The details of the architecture are described in Tab.~\ref{tbl:FashionMNISTarch}.
\begin{table}[t]
\centering
\caption{FashionMNIST architecture}
\begin{tabular}{|c|c|c|c|}
\hline
\textbf{Layer} & \textbf{Output Shape} & \textbf{Activation} & \textbf{Parameters} \\
\hline
Flatten & (784,) & - & - \\
Linear & (784, 512) & - & 401,920 \\
ReLU & (512,) & ReLU & - \\
Linear & (512, 512) & - & 262,656 \\
ReLU & (512,) & ReLU & - \\
Linear & (512, 10) & - & 5,130 \\
\hline
\end{tabular}
\label{tbl:FashionMNISTarch}
\end{table}
For Pneumonia, we use a simple CNN, again with batch normalization and max pooling, with details given in Tab.~\ref{tbl:Pneumoniaarch}.
\begin{table}[ht]
\centering
\caption{Pneumonia architecture}
\begin{tabular}{|c|c|c|c|}
\hline
\textbf{Layer} & \textbf{Output Shape} & \textbf{Activation} & \textbf{Parameters} \\
\hline
Conv2d & (3, 32, 32) & - & 896 \\
BatchNorm2d & (32, 32, 32) & - & 64 \\
Conv2d & (32, 32, 32) & - & 18,464 \\
BatchNorm2d & (64, 32, 32) & - & 128 \\
MaxPool2d & (64, 16, 16) & - & - \\
Conv2d & (64, 16, 16) & - & 36,928 \\
BatchNorm2d & (64, 16, 16) & - & 128 \\
MaxPool2d & (64, 8, 8) & - & - \\
Flatten & (4096,) & - & - \\
Linear & (2,) & - & 4,194,306 \\
\hline
\end{tabular}
\label{tbl:Pneumoniaarch}
\end{table}
For MRI we use an architecture similar to pneumonia with details described in Tab.~\ref{tbl:MRIarch}.
\begin{table}[ht]
\centering
\caption{MRI architecture}
\begin{tabular}{|c|c|c|c|}
\hline
\textbf{Layer} & \textbf{Output Shape} & \textbf{Activation} & \textbf{Parameters} \\
\hline
Conv2d & (3, 32, 32) & - & 896 \\
BatchNorm2d & (32, 32, 32) & - & 64 \\
Conv2d & (32, 32, 32) & - & 18,464 \\
BatchNorm2d & (64, 32, 32) & - & 128 \\
MaxPool2d & (64, 16, 16) & - & - \\
Conv2d & (64, 16, 16) & - & 36,928 \\
BatchNorm2d & (64, 16, 16) & - & 128 \\
MaxPool2d & (64, 8, 8) & - & - \\
Flatten & (32768,) & - & - \\
Linear & (2,) & - & 2,636,034 \\
\hline
\end{tabular}
\label{tbl:MRIarch}
\end{table}
For SVHN, we use again a standard CNN with batch normalization and max pooling, detailed in Tab.~\ref{tbl:SVHNarch}.
\begin{table}[ht]
\centering
\caption{SVHN architecture}
\begin{tabular}{|c|c|c|}
\hline
\textbf{Layer} & \textbf{Output Shape} & \textbf{Parameters} \\
\hline
Conv2d & (3, 32, 32) & 896 \\
BatchNorm2d & (32, 32, 32) & 64 \\
Conv2d & (32, 32, 32) & 9,248 \\
MaxPool2d & (32, 16, 16) & - \\
Dropout2d & (32, 16, 16) & - \\
Conv2d & (32, 16, 16) & 18,464 \\
BatchNorm2d & (64, 16, 16) & 128 \\
Conv2d & (64, 16, 16) & 36,928 \\
MaxPool2d & (64, 8, 8) & - \\
Dropout2d & (64, 8, 8) & - \\
Conv2d & (64, 8, 8) & 73,856 \\
BatchNorm2d & (128, 8, 8) & 256 \\
Conv2d & (128, 8, 8) & 147,584 \\
MaxPool2d & (128, 4, 4) & - \\
Dropout2d & (128, 4, 4) & - \\
Flatten & (2048,) & - \\
Linear & (128,) & 262,272 \\
Dropout & (128,) & - \\
Linear & (10,) & 1,290 \\
\hline
\end{tabular}
\label{tbl:SVHNarch}
\end{table}

\begin{table}[ht]
\centering
\caption{GPT2ForSequenceClassification Model Architecture}
\begin{tabular}{|c|c|}
\hline
\textbf{Layer} & \textbf{Paramters} \\
\hline
Embedding: 2-1 & 38,597,376 \\
Embedding: 2-2 & 786,432 \\
Dropout: 2-3 & -- \\
ModuleList: 2-4 & -- \\
GPT2Block: 3-1 to 3-12 & 7,087,872 (each) \\
LayerNorm: 2-5 & 1,536 \\
\hline
Linear: 1-2 & 1,536 \\
\hline
\textbf{Total params} & \textbf{124,441,344} \\
\hline
\end{tabular}
\label{tbl:GPT2arch}
\end{table}

For our experiments on interpretable models, we use $m= 5$ clients. For decision trees (DT), we split by the Gini index with at least $2$ samples for splitting. For RuleFit, we use a tree size of $4$ and a maximum number of rules of $200$. For the WineQuality dataset, we use an unlabeled dataset size of $|U|=4100$, a training set size of $136$, and a test set size of $1059$. For BreastCancer, we use an unlabeled dataset of size $|U|=370$, a training set of size $85$, and a test set of size $114$. For the AdultsIncome dataset, we use an unlabeled dataset of size $|U|=10^4$, a training set of size $31,073$, and a test set of size $7769$. For the Mushroom dataset, we use an unlabeled dataset of size $|U|=4000$, a training set of size $2499$, and a test set of size $1625$. For the covertype dataset, we use an unlabeled dataset of size $|U|=5\cdot10^4$, a training set of size $414,810$, and a test set of size $116,202$.

\section{Practical Impact of Federated Co-Training}
\label{sec:Impact}
This paper tackles the problem of privacy in collaborative and federated learning. While federated learning allows us to collaboratively train models without sharing sensitive data directly, attackers can still make non-trivial inference about sensitive data from the shared information, such as model parameters, model updates, or gradients. 

This is not only a hypothetical issue. The example of healthcare shows the necessity of collaborative training, but also the importance of ensuring privacy:
Despite isolated successes for ML in healthcare for early detection to improve outcomes and reduce costs~\citep{Keen2018,Fenton2013,Rao2007}, widespread adoption of ML to support clinical decisions remains distant.
A critical barrier is access to large and diverse patient data.
Modern ML techniques, if given access to millions of records collected in routine care from multiple institutions,
are capable of learning robust, high-performing models for disease that
identify novel markers of risk, predict disease to help clinicians intervene earlier, model disease progression and even suggest preventative interventions for individual patients~\citep{Chittora2021,Khan2020,Qin2020,Sree2020}.
Unfortunately, privacy (HIPAA) and other regulatory
and reputational concerns prevent medical institutions from sharing patient data to create large representative patient datasets for training~\citep{HIPAAOCR2012,GDPR2016,Seh2020,HIPAAJournal2019}. Federated learning improves privacy over data pooling. Its privacy vulnerability is considered a major challenge in healthcare, though~\citep{rieke2020future}, that prevents its widespread application. Our research can help resolving the privacy issues in collaborative training and thereby unlock machine learning on large, distributed, yet highly sensitive datasets, as in healthcare.
\vfill

\end{document}